\documentclass{article}

\usepackage{microtype}
\usepackage{graphicx}
\usepackage{subfigure}
\usepackage{booktabs} % for professional tables

\usepackage[np,autolanguage]{numprint}
\nprounddigits{1}

\usepackage{hyperref}

\usepackage{xspace}
\makeatletter
\DeclareRobustCommand\onedot{\futurelet\@let@token\@onedot}
\def\@onedot{\ifx\@let@token.\else.\null\fi\xspace}

\def\iid{{i.i.d}\onedot}
\def\eg{{e.g}\onedot} 
\def\ie{{i.e}\onedot}

\makeatother

\usepackage{tcolorbox}

\usepackage{lipsum}

\usepackage[accepted]{icml2024}

\usepackage{amsmath}
\usepackage{amssymb}
\usepackage{mathtools}
\usepackage{amsthm}

\usepackage[capitalize,noabbrev]{cleveref}

\theoremstyle{plain}
\newtheorem{theorem}{Theorem}[section]

\newtheorem{lemma}[theorem]{Lemma}

\theoremstyle{definition}

\theoremstyle{remark}

\usepackage[textsize=tiny]{todonotes}

\newcommand{\Title}{More Flexible PAC-Bayesian Meta-Learning by Learning Learning Algorithms}
\icmltitlerunning{\Title}

\newcommand{\Z}{\mathcal{Z}} 
\newcommand{\F}{\mathcal{F}}
\newcommand{\M}{\mathcal{M}} 
\newcommand{\A}{\mathcal{A}}
\newcommand{\T}{\mathcal{T}}
\newcommand{\qa}{\mathcal{Q}(A)}
\newcommand{\q}{\mathcal{Q}} 
\newcommand{\pa}{\mathcal{P}(A)}
\newcommand{\p}{\mathcal{P}}
\newcommand{\KL}{\operatorname{\mathbf{KL}}}
\newcommand{\er}{\mathcal{R}} 
\newcommand{\her}{\widehat{\mathcal{R}}} 
\newcommand{\ter}{\widetilde{\mathcal{R}}} 
\newcommand{\E}{\operatorname*{\mathbb{E}}}
\newcommand{\prior}{\mathfrak{P}(\pi, \p)}
\newcommand{\posterior}{\mathfrak{Q}(\rho, \q)}

\newcommand{\pri}{\mathfrak{P}(\p)}
\newcommand{\post}{\mathfrak{Q}(A, \qa)}
\newcommand{\R}{\mathbb{R}} 
\usepackage{bbm}
\newcommand{\ind}{\mathbbm{1}} 
\newcommand{\w}{\mathbf{w}}
\newcommand{\x}{\mathbf{x}} 
\newcommand{\N}{\mathcal{N}} 
\newcommand{\U}{\mathcal{U}}

\renewcommand{\paragraph}[1]{\medskip\noindent\textbf{#1}\quad}

\begin{document}

\twocolumn[
\icmltitle{\Title}



\icmlsetsymbol{equal}{*}

\begin{icmlauthorlist}
\icmlauthor{Hossein Zakerinia}{ista}
\icmlauthor{Amin Behjati}{sharif}
\icmlauthor{Christoph H. Lampert}{ista}
\end{icmlauthorlist}

\icmlaffiliation{ista}{Institute of Science and Technology Austria (ISTA)}
\icmlaffiliation{sharif}{Sharif University of Technology}

\icmlcorrespondingauthor{Hossein Zakerinia}{Hossein.Zakerinia@ist.ac.at}

\icmlkeywords{Machine Learning, ICML}

\vskip 0.3in
]



\printAffiliationsAndNotice{}  % leave blank if no need to mention equal contribution

\begin{abstract}
We introduce a new framework for studying 
meta-learning methods using PAC-Bayesian theory. 
Its main advantage over previous work is that
it allows for more flexibility in how the 
transfer of knowledge between tasks is realized. 
For previous approaches, this could only happen
indirectly, by means of learning \emph{prior 
distributions over models}. 
In contrast, the new generalization bounds 
that we prove express the process of 
meta-learning much more directly as learning 
the \emph{learning algorithm} that should be 
used for future tasks.
The flexibility of our framework makes it suitable to analyze a wide range of meta-learning mechanisms and even design new mechanisms. Other than our theoretical contributions we also show empirically that our framework improves the prediction quality in practical meta-learning mechanisms.
\end{abstract}

\section{Introduction}
Machine learning systems have remarkable success 
in solving complex tasks when they are trained on large amounts of data. 
However, their success is still limited when only little data is available for a task.
One reason for this is that common machine learning algorithms, 
such as minimizing training loss through gradient-based optimization,
are very generic. 
Tailored to be applicable across a wide range of data sources and target tasks, the underlying models need to have many degrees of freedom, and they require a lot of training data
to adjust these suitably.
In contrast to this, natural learning systems can learn new tasks from little task-specific data. They achieve this by transferring and reusing information from their past experience to new tasks, 
instead of learning a new model from scratch every time. 

\emph{Meta-learning} (also called \emph{learning-to-learn}) 
is a principled way of also giving machine learning systems the 
ability to share knowledge between related learning tasks~\citep{schmidhuber1987evolutionary,thrun1998}. 
Instead of learning individual models for each task, a meta-learner learns a mechanism, a \emph{learning algorithm}, that sets the model parameters given a (usually small) amount of training data. 

In practice, there are numerous possibilities to realize this step.
\emph{Model-based} approaches to meta-learning learn prototypical models which can efficiently be adjusted (\eg finetuned) to specific tasks~\citep{finn2017model,nichol2018first}. 
Relatedly, \emph{regularization-based} approaches learn regularization terms that prevent future tasks from overfitting even if trained on little data~\citep{denevi2019learning}. 
\emph{Hypernetwork}-based approaches learn secondary networks that output the weights of task-specific models~\citep{zhao2020meta,scott2023pefll}.
\emph{Representation-based} approaches learn (often low-dimensional) feature representations in which learning can be performed with less training data than in the original input space~\citep{maurer2009transfer,maurer2016benefit,lee2019meta}.
\emph{Optimization-based} approaches learn the steps or (hyper)parameters of an optimization procedure~\citep{hochreiter2001learning,ravi2016optimization,li2017meta}. 

All meta-learning methods strive for \emph{generalization} between previously seen tasks and future ones. 
Unfortunately, most of the above methods are only understood in terms of their empirical performance on example tasks, but they lack theoretical guarantees on their generalization abilities.

\emph{Meta-learning theory} studies the theoretical properties of meta-learning method. 
In particular, it aims at providing quantitative generalization guarantees for them, in the form of high-probability upper bounds on the quality of models learned on future tasks, even before data for these tasks is available. 
Historically, the first attempts to do so had the form of classical PAC guarantees, which were data-independent and had to be derived individually for each method~\citep{baxter2000model,maurer2009transfer}.
However, starting with~\citet{pentina2014pac}, most recent works exploit PAC-Bayesian theory,
which proved to be more flexible and powerful. 
Not only does it allow deriving bounds that can be instantiated for different meta-learning
methods, but the generalization guarantees are also data-dependent, and interpreting them as 
learning objectives can inspire new meta-learning methods~\citep{amit2018meta, rothfuss2021pacoh}.
Unfortunately, existing theoretical results still apply only to a small number of actual meta-learning methods, namely those in which the knowledge transfers from previous to future tasks can be expressed as a shift of prior distributions over models.
Of the ones we mention above, that includes methods based on learning
a model prototype or a regularization term, but not the more flexible 
ones based on learning representations, optimizers, or hypernetworks.

\textbf{Our main contribution in this work is the introduction of new 
form of PAC-Bayesian generalization bounds that provides theoretical 
guarantees for a much broader class of meta-learning methods than previous
ones.} In particular, this includes \emph{all} types of methods listed above.
Specifically, the transfer of knowledge from previous to future tasks 
is modeled not indirectly, using distributions over models priors, but 
directly, using distributions over \emph{learning algorithms}.
This viewpoint reflects the \emph{learning-to-learn} aspect of meta-learning much better than previous ones, as it allows the meta-learner to directly select which learning algorithms are meant to be executed on future tasks. In contrast, the prior-based transfer of previous bounds allows only an indirect influence by means of expressing a preference of some models over others.

Besides our theoretical contributions, we also report on experiments in two standard benchmark settings: we demonstrate that using our generalization bound as a learning objective yields a meta-learning algorithm of improved empirical performance compared to previous methods based on prior-based transfer, and we show that even in the case where knowledge transfer can actually be expressed as model priors, our bound is numerically tighter than previous ones. 

\section{Background}\label{sec:background}
There are many ways how the sharing of information between learning tasks can be formalized. 
In this work, we adopt the meta-learning setting first proposed in \citet{baxter2000model} under the name of \emph{learning to learn}.
We call a tuple $t=(D,S)$ a \emph{task}, where $D$ is a data distribution over a sample space $\Z$, and $S=\{z_1,\dots,z_m\}$ is a dataset sampled \iid according to this distribution. 
A \emph{meta-learning method} (or \emph{meta-learner}) has access to the training sets, $S_1,\dots,S_n$ of a number of \emph{training tasks}, $t_1,\dots,t_n$, which themselves are \iid sampled from an unknown distribution, $\tau$, over a \emph{task environment}.
Let $\F$ be a hypothesis set of possible models, and $\ell:\Z\times\F\to[0,1]$ a loss function. 
For any set $X$, we denote by $\M(X)$ the set of probability distributions over $X$, and by $\mathcal{P}(X)$ the power set of $X$, where for our purposes, only subsets of finite size matter.

The goal of meta-learning is to output (in 
some parameterized form) a \emph{learning algorithm}, 
$A:\mathcal{P}(\Z)\to\F$, \ie, a mapping 
from the set of datasets to the set of models, with 
the goal to make the \emph{risk} (expected loss) as 
small as possible in expectation when applying the 
algorithm $A$ to a future task, \ie minimize 
\begin{align}
\E_{(D,S)\sim\tau} \er(A(S))
\quad\text{for}\quad \er(f)=\E_{z\sim D}\ell(z,f).
\end{align}
Our main tool for analyzing such meta-learning methods
theoretically will be PAC-Bayesian learning theory. 
Before we apply this in the meta-learning setting, 
we remind the reader of its main concepts in the 
standard setting.

\subsection{PAC-Bayesian Learning}
Classical PAC-Bayesian bounds~\citep{mcallester1998some, maurer2004note} quantify the generalization properties of \emph{stochastic models}. A stochastic model is parameterized by a distribution, $Q\in\M(\F)$, over the model space. For any input $z\in\Z$ it makes a stochastic predictions by sampling $f\sim Q$ and outputting $f(z)$.
We extend the loss function to this as  
\begin{align}
\ell(z,Q)&:=\E_{f \sim Q}\ell(z,f).\label{eq:er_zQ}
\end{align}
Now, let $t=(D,S)$ be the given task. % that should be learned.
The PAC-Bayesian framework provides an upper bound 
on the \emph{expected risk} of a (stochastic) model, $Q$, 
\begin{align}
\er(Q) &=\E_{z \sim D} \ell(z,Q)
\intertext{in terms of its \emph{empirical risk}}
\her(Q) &=\frac{1}{|S|}\sum_{z\in S} \ell(z,Q),
\end{align}
and some complexity terms.

\iffalse
The \emph{expected risk} and the \emph{empirical risk} 
of a (stochastic) model $Q$ are 
%
\begin{align}
\er(Q) &=\E_{z \sim D} \er(z,Q)
\intertext{and}
\her(Q) &=\frac{1}{m}\sum_{i=1}^m \er(z_i,Q),
\end{align}
respectively. 

The PAC-Bayesian framework provides an upper bound for 
$\er(Q)$ based on $\her(Q)$ and a complexity term.
\fi 

For example, from \citet{maurer2004note} it follows that for 
any fixed $\delta > 0$ and any fixed prior distribution over 
models $P \in \M(\F)$, 
with probability at least $1-\delta$ over the sampling of the 
training dataset, $S$, it holds that for all $Q\in\M(\F)$,
\begin{align}
    \er(Q) \leq \her(Q) +\sqrt{ \frac{\KL(Q\|P) + \log(\frac{2\sqrt{m}}{\delta})}{2m}},
    \label{eq:maurer}
\end{align}
where $\KL$ denotes the Kullback-Leibler divergence. In words, 
the stochastic model $Q$ is guaranteed to generalize well, if 
it is chosen sufficiently close to the prior, $P$. 

Since their introduction by \citet{mcallester1998some}, many similar 
bounds have been derived that mostly differ the way they compare 
empirical and expected error, the specific form of the complexity 
term, and with further assumptions they make.
See, , \eg, \citet{guedj2019primer, alquier2021user, hellstrom2023generalization}
for surveys. 
However, the bounds have in common that the size of the complexity term 
is mostly determined by the $\KL$ divergence between the posterior 
distribution $Q$ and a fixed data-independent prior $P$, like it 
does in Equation~\eqref{eq:maurer}.
The bounds also have in common that they hold uniformly with 
respect to $Q$. This means that one can use the right-hand side 
of the inequality as a training objective, and the guarantees will 
still hold for the (stochastic) model resulting from minimizing it.

\subsection{PAC-Bayesian Meta-Learning}\label{subsec:meta}
PAC-Bayesian bounds for meta-learning were pioneered by \citet{pentina2014pac}.
They assumed a fixed learning procedure that outputs a posterior distribution
over models, $Q(S,P)\in\M(\F)$, depending on the training data, $S$, as well 
as on a prior distribution, $P\in\M(\F)$.
A canonical example of such a procedure would be to return the stochastic model that minimizes the right hand side of~\eqref{eq:maurer}. 

For any prior, $P$, the expected risk on a new task,
\begin{align}
    \er(P)=\E_{(S,D)\sim\tau}\E_{z\sim D}\ell(z,Q(S,P)), \label{eq:erP}
\end{align} 
provides a measure how suitable this choice of prior is for future tasks.
However, \eqref{eq:erP} cannot be computed, 
because it depends on unobserved quantities. 
Under Baxter's assumptions that the available training tasks 
are sampled from the same task environment as future tasks, 
the \emph{empirical multi-task risk}, 
\begin{align}
\her(P)=\frac1n\sum_{i=1}^n \frac{1}{|S_i|}\sum_{z\in S_i}\ell(z,Q(S_i,P)), \label{eq:herP}
\end{align}
can serve as empirical proxy for $\er(P)$.
The difference between \eqref{eq:erP} and \eqref{eq:herP} 
can be bounded with PAC-Bayesian techniques, as long as 
the prior is not simply chosen in a deterministic way, 
but by means of specifying its posterior distribution, 
$\q\in\M(\M(\F))$, called the \emph{hyper-posterior}.
Overall, one obtains guarantees that $\er(\q)=\E_{P\sim\q}\er(P)$
is bounded by $\her(\q)=\E_{P\sim\q}\her(P)$ and some complexity
terms that are increasing functions of $\KL(\q\|\p)$, where 
$\p\in\M(\M(\F))$ is a data-independent \emph{hyper-prior}, 
and of $\KL(A(S_i,P)\|P)$ in expectation over $P\sim\q$.

Numerous later works improved and extended these PAC-Bayesian meta-learning bounds:
\citet{amit2018meta} designed an optimization algorithm based on this 
setup for neural networks, 
\citet{liu2021pac} proved bounds with a different form,
\citet{guan2022fast} and \citet{riou2023bayes}
proved fast rate bounds for this setup,
\citet{friedman2023adaptive} proved bounds based on data-dependent PAC-Bayes bounds,
and \citet{rezazadeh2022unified} provided a general framework for proving several different 
form bounds.
\citet{rothfuss2021pacoh, rothfuss2022pacoh} generalized the setup to unbounded 
loss functions, and developed a new algorithm for estimating optimal hyper-posteriors. 
\citet{ding2021bridging, tian2023can} studied this setup for few-shot meta-learning, and
\citet{farid2021generalization} studied the connection between PAC-Bayes and uniform stability in this setup.

While these works constitute substantial progress, all of them share a 
common limitation that they inherited from the setup originally defined 
in \citet{pentina2014pac}:
\textbf{they only apply to meta-learning methods that are expressible  
as a single learning strategy parameterized by a prior distribution 
over models.}
However, many practical meta-learning algorithms do not follow this pattern, 
thereby preventing the existing PAC-Bayesian frameworks from studying the 
generalization ability of these algorithms. 

One exception is \citet{pentina2015lifelong}, which provided a bound over transformation operators between tasks, but this situation applies only under rather restrictive assumptions. 
Another one is the recent \citet{scott2023pefll}, which proved a 
related PAC-Bayesian bound 
in the context of personalized federated learning. 
However, the result provides only rather weak guarantees, because it assumes fixed prior distributions that have to be chosen without any knowledge about the task environment instead of benefiting from environment-dependent priors as the works above.
Moreover, ~\citet{nguyen2022pac} provided a bound in an alternative meta-learning framework that uses both validation and training data, which also assumes fixed priors.

As an alternative framework, information-theoretic bounds have been derived~\citep{chen2021generalization, hellstrom2022evaluated,hellstrom2023generalization}.
These, however, typically provide bounds in expectation rather than with high probability 
over the training tasks, and they are harder to compute than the PAC-Bayesian ones.
Additionally, these works use distribution-dependent priors
while the works in the PAC-Bayesian framework use data-dependent priors (through a hyper-posterior).
In Section~\ref{sec:discussion}, we describe the role of hyper-posteriors in more detail.
 
In this work, \textbf{we introduce a new form of PAC-Bayesian meta-learning bounds that overcomes the limitation of previous works.} It works in a more general setup that applies to any set of learning algorithms as well as allowing for algorithm-specific hyper-posteriors. 

\begin{table}
\centering
\caption{Notations}\smallskip
\begin{tabular}{|c|c|}
\hline
$P \in \M(\F)$ &  Prior distribution \\
\hline
$Q \in \M(\F)$ & Posterior distribution \\
\hline
$\pi \in \M(\A)$ & Meta-Prior over algorithms  \\
\hline
$\rho \in \M(\A)$ & Meta-Posterior over algorithms  \\
\hline
$\pa \in \M(\M(\F))$ & Hyper-Prior over priors \\
\hline
$\qa \in \M(\M(\F))$ & Hyper-Posterior over priors  \\
\hline
\end{tabular}
\end{table}

\section{Main Results} \label{sec:results}
In this section, we state and discuss our main result: 
a generalization bound that holds for any meta-learning 
method that is expressible as a way to \emph{choose a 
learning algorithm for future tasks}. 

Formally, let $\A = \{A: \mathcal{P}(\Z)\to\M(\F)\}$
be a \emph{set of stochastic learning algorithms} that 
take as input a dataset and output a posterior distribution 
over models.
Note that this algorithm set does not have to be homogeneous. 
For example, $\A$, could contain different architectures of 
neural networks, which are initialized in different ways and 
adjusted to the training data by different optimizers, or decision 
trees with different construction rules, or support vector machines 
with different kernels, or prototype-based classifiers with 
different distance measures, or all of the above together.

Given a set of training datasets, $S_1,\dots,S_n$, the meta-learner outputs a posterior distribution over the algorithms, $\rho\in\M(\A)$. 
We call $\rho$ the \emph{meta-posterior (distribution)}, and we define its risk on future tasks as
\begin{align}
    \er(\rho)= \E_{A \sim \rho} \E_{(S,D)\sim\tau}\E_{z\sim D}\ell(z,A(S)).\label{eq:erRho}
\end{align}
If this value is small then the meta-learner has done a good 
job at identifying learning algorithms that work well on 
future tasks.
As such, Equation~\eqref{eq:erRho} describes the actual 
quantity of interest. However, it is not a computable value.
Therefore, we introduce the \emph{empirical risk} of the 
meta-posterior, $\rho$, on the $n$ training tasks as 
\begin{align}
\her(\rho)=\E_{A \sim \rho}\frac1n\sum_{i=1}^n \frac{1}{|S_i|}\sum_{z\in S_i}\ell(z,A(S_i)).
\label{eq:herRho}
\end{align}
Our main results, Theorems~\ref{theorem:Main_pa} and~\ref{theorem:Main_p} 
below, provide upper bounds on $\er(\rho)$ in terms of $\her(\rho)$
and suitable complexity terms.

Before stating them, we introduce one additional 
source of flexibility that our framework possesses.
Remember that in the classical setting of 
Section~\ref{subsec:meta}, one fixed hyper-prior 
distribution over priors was given, and the 
meta-learner was meant to learn one 
hyper-posterior distribution over priors. 
In our setting, especially if the algorithm 
set is heterogeneous, it stands to reason that 
different algorithms might benefit from 
different choices of prior distributions. 
To express this, let $\p:\A\to\M(\M(\F))$ 
now be a fixed data-independent mapping of algorithms to 
hyper-priors, \ie for each algorithm, $A$, 
one hyper-prior, $\pa\in\M(\M(\F))$, is 
associated in a data-independent way.
Analogously, denote by $\q:\A\to\M(\M(\F))$ 
a mapping of algorithms to hyper-posteriors.
As part of the meta-learning process, the meta-learner 
constructs $\q$ by specifying a hyper-posterior 
distribution, $\qa\in\M(\M(\F))$, for any 
learning algorithm $A\in\A$. 

We now state our first main result: a generalization
bound for $\er(\rho)$ in terms of $\her(\rho)$
that holds with high probability uniformly 
for all possible choices of $\rho$ and $\q$. 
\begin{theorem}
For any fixed meta-prior $\pi$, fixed hyper-prior mapping 
$\p$ and any $\delta>0$, with probability 
at least $1-\delta$ over the sampling of the training tasks,
for all distributions $\rho \in\M(\A)$ over algorithms,
and for all hyper-posterior mappings $\q:\A\to\M(\M(\F))$ 
it holds
\begin{align}
    &\er(\rho) \leq \her(\rho) +\sqrt{ \frac{\KL(\rho\|\pi) + \log(\frac{4\sqrt{n}}{\delta})}{2n}}  
\label{eq:bound_1}
    \\& + \sqrt{\frac{\KL(\rho||\pi) + \E_{A \sim \rho}[C_1(A, \q, \p)]+ \log(\frac{8mn}{\delta}) + 1}{2mn}},
\notag
\end{align}
with 
\begin{align}
\begin{split}
C_1(A, \q, \p)=&\KL(\qa \| \pa) 
\\&+ \E_{P \sim \qa}  \sum_{i=1}^{n} \KL(A(S_i) || P).
\label{complexity_1}
\end{split}
\end{align}
\label{theorem:Main_pa}
\end{theorem}

Our second main result is a tightened variant of Theorem~\ref{theorem:Main_pa}
that holds in the special case that all algorithms share the same hyper-prior, \ie $\p$ is constant.

\begin{theorem}
For any fixed meta-prior $\pi$, fixed hyper-prior $\p$ and any $\delta>0$ 
with probability at least $1-\delta$ over the sampling of the datasets, 
for all distributions $\rho \in\M(\A)$ over algorithms,
and for all hyper-posterior functions $\q: \A \rightarrow \M(\M(\F))$
it holds 
\begin{align}
    \er(\rho) &\leq  \her(\rho) + \sqrt{ \frac{\KL(\rho||\pi) + \log(\frac{4\sqrt{n}}{\delta})}{2n}}  
    \\& \quad+  \ \E_{A \sim \rho} \sqrt{\frac{C_2(A, \q, \p)+ \log(\frac{8mn}{\delta}) + 1}{2mn}}
\label{eq:bound_2}
\end{align}
\begin{align}
C_2(A, \q, \p) &= \KL(\qa \| \p) 
\\&\quad+ \E_{P \sim \qa}  \sum_{i=1}^{n} \KL(A(S_i) || P) \nonumber
\label{eq:complexity_2}
\end{align}
\label{theorem:Main_p}
\end{theorem}
We provide proof sketches for both theorems in Section~\ref{sec:proofsketch}.
For complete proofs, see Appendix~\ref{app:proofs}.

\subsection{Discussion} \label{sec:discussion}
In this section, we discuss the properties of the bounds 
and explain the role and benefits of different terms. 
We also highlight the differences between our general setup 
with the more narrow setups of previous works and explain 
how our algorithm applies to existing meta-learning methods.

\paragraph{Complexity terms}
As it is common for PAC-Bayesian meta-learning, the bounds~\eqref{eq:bound_1} 
and \eqref{eq:bound_2} each contain two complexity terms, 
which reflect the two types of generalization required for 
meta-learning guarantees: \emph{task-level generalization}, 
\ie generalization from the observed tasks to future tasks, 
and \emph{within-task generalization} for all training 
tasks (also called \emph{multi-task generalization}).
In the following, we first discuss~\eqref{eq:bound_1} 
in detail and afterwards discuss how~\eqref{eq:bound_2} 
differs from it. 

The first complexity term of~\eqref{eq:bound_1} expresses
the aspect of task-level generalization: it contains the 
$\KL$-divergence between the data-dependent 
meta-posterior and the data-independent meta-prior over 
algorithms. As such, it reflects directly how much the 
choice of learning algorithm is influenced by the data.
In addition, it contains an additional logarithmic term 
that is small for all practical choices of $n$ and $\delta$. 
Both terms are divided by $2n$, meaning that the first 
complexity term decreases with the number of training 
tasks and vanishes (only) for $n\to\infty$. 
It is not affected by the number of training samples per 
task, $m$.
Such a behavior makes sense: the uncertainty about the
task environment, \ie what kind of tasks will appear in 
the future, is reduced with each additional training task, 
but having more data points from the tasks available 
does not provide new information about this aspect. 

The second complexity term of Equation~\eqref{eq:bound_1} 
contains the same $\KL$-divergence term as well as two
additional ones: in simplified form (ignoring the 
expectations over algorithms), the first term relates 
the algorithm's hyper-posterior, $\qa$, to its 
hyper-prior, $\pa$.
This term reflects the amount by which the meta-learner's 
choice of priors depends on the observed data. 
However, the denominator for this complexity term is $2nm$
instead of $2n$ for the first term, indicating that the
hyper-posterior can be adjusted rather flexibly without 
increasing the size of this term too much. 
The final $\KL$-term relates each task's predicted models, 
$A(S_i)$, and its respective prior distributions, $P$.
The sum of these $n$ terms is divided by $2nm$, making the 
term decrease with $m$ and vanish for $m\to\infty$. 
Once again, a term of this type makes sense. It reflects 
the average uncertainty about the true risk for models 
learned from finite data of each training task.
When the number of samples, $m$, per task grows, the 
uncertainty about each task is reduced. 
When just the number, $n$, of training tasks grows, 
however, the amount of data per task remains the 
same, so no reduction of the average per-task uncertainty 
can be expected. 
The remaining terms in the numerator depend only 
logarithmically on this number and $\delta$ and
are negligible in most practical settings. 

More precisely, Equation~\eqref{eq:bound_1} contains 
the expectations of these terms over the actually 
stochastic choice of algorithm and prior distribution.

\paragraph{Hyper-posteriors}
A basic PAC-Bayes bound with a fixed prior would result in separate and independent complexity terms for each task, 
independent of the environment, and will not take into account the relation between training tasks.
Instead, we introduce algorithm-dependent hyper-posteriors, from which we sample priors, 
and are learned specifically for each learning algorithm, shared between all the tasks. 
Therefore, the $n$ complexity terms for each $A$ become $\E_{P \sim \qa} \sum_{i=1}^{n} \KL(A(S_i) || P)$ with the additional cost of $\KL(\qa \| \pa)$,
which improves $n$ terms at the cost of one extra term.

In this formulation, $\qa$ can be seen as a similarity measure for the output of the algorithm. 
The complexity measure is small if for the outputs of the algorithm for $n$ tasks, 
there is a good hyper-posterior to generate priors close to all posteriors. 
Note that we can learn different hyper-posteriors for different algorithms,
and capture these similarities specifically for the outputs of each algorithm.

Note that the hyper-posterior is a mathematical notion,
and the bound holds for all hyper-posteriors at the same time (with high probability).
Its role is to help better capture the relations between tasks when using a specific algorithm.
For a future task, only the meta-posterior would apply and the role of the hyper-posterior by default is implicit.

\paragraph{Difference between the theorems}
The bound~\eqref{eq:bound_2} differs from bound~\eqref{eq:bound_1} 
most in the fact that the term $\KL(\rho\|\pi)$ does not appear 
in the second term, and $\rho$ only appears in the second term 
as the distribution over algorithms (when we take the expectation 
of algorithms in $\E_{A \sim \rho}$). 
The expectation is also moved outside of the square root, which 
makes the bound tighter (Jensen's inequality). 
We attribute the differences mainly to the fact that for an 
algorithm-independent hyper-prior, some steps in the proof
can be performed in a tighter way.
Generally, both bounds agree in their main behavior with 
respect to the number of training tasks and samples.

\paragraph{Comparison with previous works}
The setting of Section~\ref{sec:results} is a strict 
generalization of the setup from previous works where
the meta-learner only learned priors. 
In fact, the latter setting can be recovered from ours
as follows:
let $Q(S, P)$ be the fixed learning rule of a prior-based
meta-learning method. We then define a family of algorithms 
as $\A=\{ Q(\cdot,P) : P\in\M(\F) \}$. With each element of 
$\A$ uniquely determined by a prior $P$, choosing a 
distribution of algorithms is equivalent to choosing 
a distribution over priors. 
Now, by setting $\qa  = \q(P) = \delta_P$ and $\pa  = \p(P) = \delta_P$,
Theorem \ref{theorem:Main_pa} is applicable.
In Section \ref{sec:numerical_comparison} we compare the 
resulting generalization bound numerically to prior ones
in this setting.

Similarly, we can recover the bounds of \citet{scott2023pefll}, 
which transfer algorithms but only allow for a single fixed choice of prior:
for any $A\in\A$ we set $\qa  = \delta_{P_0}$ and $\pa = \delta_{P_0}$,
where $P_0$ is the fixed prior. Again, Theorem \ref{theorem:Main_p} is
now applicable.
This construction also shows that our result not only recovers 
the bound of \citet{scott2023pefll} but improves over it.
The reason is that Theorem~\ref{theorem:Main_p} holds uniformly over 
all (potentially data-dependent) choices of $\q$, of which the construction 
described above is simply a single data-independent choice.

\paragraph{Recovering common meta-learning methods}
As discussed in the introduction, previous works on meta-learning 
that rely on transferring priors over models are not applicable to 
hypernetwork-based, representation-based, or optimization-based 
meta-learning methods because these require different parametrizations 
of their algorithm sets. 

In our framework, expressing these methods is straight-forward.
For the hypernetwork-based methods~\citep{zhao2020meta,scott2023pefll}, 
the set of algorithms is parametrized by the set of hypernetwork weights. 
Consequently, learning the algorithm means training the hypernetwork. 
For representation-based methods~\citep{maurer2009transfer,maurer2016benefit,lee2019meta} parametrizing each algorithm is a feature extractor, such as a linear 
projection or a convolutional network. 
For optimization-based meta-learning~\citep{hochreiter2001learning,ravi2016optimization,li2017meta},
the algorithm set can contain all hyperparameters to be learned,
or the set of all considered optimization procedures, 
\eg in the form of the weights of a recurrent network. 
In all cases, our bound can directly be applicable. 
In fact, based on our bound one might even improve such methods 
by suggesting appropriate (meta-)regularization term.

Another observation is that our framework also allows 
combining different approaches. 
For example, the algorithm set could be parametrized by the starting 
point of an optimization step (\eg the initialization of a network), 
as well as by a regularization term.
The result is a hybrid of methods based on model prototypes and 
on methods based on learning a regularizer.
Prior works were applicable to study either of these approaches 
individually, but not their combination.
Nevertheless, we provide an experimental demonstration that such 
a hybrid approach can be beneficial in Section~\ref{sec:experiments}.

\section{Proof Sketch}\label{sec:proofsketch}
We provide the proof sketch of Theorem \ref{theorem:Main_pa}. 
The full proof and the proof of Theorem \ref{theorem:Main_p} are available in Appendix \ref{app:proofs}.

For the proof we first define an intermediate objective
that represents the true risk of the training tasks,
\begin{align}
    &\ter(\rho) = \E_{A \sim \rho} \frac{1}{n} \sum_{i=1}^{n}\E_{z \sim D_i} \ell (z,A(S_i)).
\end{align}
The proof is then divided into two parts. 
First, 
we bound the difference of the true risks between training tasks and future tasks $\er(\rho) - \ter(\rho)$.
Second, we bound the difference between the true risk and empirical risk of training tasks $\ter(\rho) - \her(\rho)$.
The final result follows by combining the two bounds.

\paragraph{Part I} To bound $\er(\rho) - \ter(\rho)$ we use standard PAC-Bayesian arguments, specifically the following lemma:
\begin{lemma}\label{lemma:part1}
For all $\delta>0$ it holds with probability at least $1 - \frac{\delta}{2}$ over
the sampling of tasks for all meta-posteriors $\rho \in \M(\A)$:
\begin{align}
    &\er(\rho) - \ter(\rho) \le \sqrt{ \frac{\KL(\rho||\pi) + \log(\frac{4\sqrt{n}}{\delta})}{2n}}.
\end{align}
\end{lemma}
This lemma is an application of applying the standard PAC-Bayesian bounds \citet{maurer2004note, perez2021tighter}.
For details, see Appendix \ref{app:proofs}.

\paragraph{Part II} We define the following two functions which produce distributions over $\A\times\M(\F)\times \F^{\otimes n}$, \ie they assigns joint probabilities to tuples, 
$(A, P, f_1, ..., f_n)$, which contain an algorithm, a prior over models, and $n$ models.
%:
%

\emph{Posterior $\posterior$:} given as input a meta-posterior $\rho$ over algorithms
    and a hyper-posterior mapping $\q$ as input, it outputs the distribution over 
    $\A\times\M(\F)\times \F^{\otimes n}$ with the following generating process:
    \emph{i)} sample an algorithm $A\sim\rho$, \emph{ii)} sample a prior $P\sim\qa$, 
    \emph{iii)} for each task, $i=1,\dots,n$, sample a model $f_i\sim A(S_i)$. 

\emph{Prior $\prior$:} given as input a meta-prior $\pi$ over algorithms and a hyper-prior mapping $\p$ as input, it outputs the distribution over $\A\times\M(\F)\times \F^{\otimes n}$ with the following generating process: \emph{i)} sample an algorithm $A\sim \pi$, 
\emph{ii)} sample a prior $P\sim\pa$, \emph{iii)} for each task, $i=1,\dots,n$, sample 
a model $f_i\sim P$.

Note that the inputs to $\posterior$ are data-dependent and will be learned from data.
In contrast, the input to $\prior$ are data-independent and need to be fixed before seeing the data.

With these definitions, we state the following key lemma:
\begin{lemma}\label{lemma1:part2}
For any fixed meta-prior $\pi$, fixed hyper-prior function $\p$,  
any $\delta > 0$ and any $\lambda>0$, it holds with 
probability at least $1 - \frac{\delta}{2}$ over the 
sampling of the training datasets that for all 
meta-posteriors $\rho \in\M(\A)$ over algorithms, and 
for all hyper-posterior functions $\q: \A \rightarrow \M(\M(\F))$:
\begin{align}
\begin{split}
         \ter(\rho) - \her(\rho)
          \le &\  \ \frac{1}{\lambda} (\KL(\posterior \| \prior)) \\&
           +    \frac{1}{\lambda} \log(\frac{2}{\delta}) +  \frac{\lambda}{8nm}
\end{split}
\end{align}
\end{lemma}

\begin{proof}
First for any task $i$ and any model $f_i$ we define:
\begin{align}
    \Delta_i(f_i) = \E_{z \sim D_i} \ell (z,f_i) - \frac{1}{|S_i|} \sum_{z\in S_i} \ell (z,f_i).
\end{align} 
By this definition and the definitions of $\ter$ and $\her$ we have:
\begin{align}
    \E_{(A, P, f_1, \dots, f_n) \sim \posterior} &\Big[\frac{1}{n} \sum_{i=1}^{n} \Delta_i(f_i) \Big] = \ter(\rho) - \her(\rho) 
\end{align}

Using this equation and the \emph{change of measure inequality}~\citep{seldin2012pac}
between the two distributions $\posterior$ and $\prior$, for any $\lambda>0$, any $\rho$ and any $\q$, we have:
\begin{align}
\begin{split}
        &\ter(\rho) - \her(\rho) - \frac{1}{\lambda}  \KL(\posterior || \prior)
        \\& \le \frac{1}{\lambda} \log \E_{(A, P, f_1, \dots, f_n) \sim \prior} \prod_{i=1}^{n} e^{\frac{\lambda}{n} \Delta_i(f_i)}
        \label{eq:change_of_measure}
\end{split}
\end{align}
It remains to bound the right-hand of~\eqref{eq:change_of_measure}. Given that $\pi$ and $\p$ are data-independent, standard tools (in particular Hoeffding's lemma and Markov's inequality) allow us to prove an upper bound that holds in high probability with respect to the randomness of training datasets, from which the statement in the lemma follows.
Detailed steps are provided in Appendix \ref{app:proofs}.
\end{proof}

The following lemma provides a split of the $\KL$-term from Lemma~\ref{lemma1:part2}.
\begin{lemma}
    \label{lemma:KL}
    For the posterior and prior defined above we have:
    \begin{align}
        & \KL(\posterior || \prior) =  \KL(\rho \| \pi)  
        \\& + \E_{A \sim \rho} \Bigg[\KL(\qa\| \pa) +\!\! \E_{P \sim \qa}  \sum_{i=1}^{n} \KL(A(S_i) || P)\Bigg]
\notag
\end{align}
\end{lemma}
The proof makes use of the explicit construction of $\prior$
and $\posterior$. It can be found in Appendix \ref{app:proofs}.

\paragraph{Proof of Theorem \ref{theorem:Main_pa}}
To get tight guarantees, we need to choose the value of $\lambda$ in 
Lemma~\ref{lemma1:part2} an optimal way dependent on the data. As
the statement of the Lemma is not uniform in $\lambda$, we do so 
approximately by allowing a fixed set of values in the range $\{1, ..., 4mn\}$ 
and applying a union-bound argument for values in this set.
The theorem then follows by combining the result with 
Lemma~\ref{lemma:part1} and using Lemma~\ref{lemma:KL}.

\paragraph{Proof sketch of Theorem \ref{theorem:Main_p}}
The proof is similar to the proof of Theorem \ref{theorem:Main_pa}. 
For the first part, we use the same Lemma~\ref{lemma:part1}.
For the second part, we use the fact that we have the same 
data-independent prior for all algorithms. 
Due to this fact, we can remove $\rho$ in the posterior 
function and prove a generalization bound that holds uniformly 
for all algorithms applied to the datasets.
Therefore we can bound the multi-task generalization of 
all meta-posteriors $\rho$ without the term $\KL(\rho\|\pi)$, 
and the result is Theorem~\ref{theorem:Main_p}. 
For the detailed proof, please see Appendix \ref{app:proofs}.

\section{Experimental Demonstration} \label{sec:experiments}
While our main contribution in this work is theoretical,
We also report on two experimental studies 
that allow us to better relate our results to prior work.

\subsection{Numerical Evaluations of the Bound} \label{sec:numerical_comparison}

In this section, we numerically compare the tightness 
of our bound to those from prior work, as far as this
is possible.
We adopt the same scenario as \citet{rothfuss2022pacoh}, 
in which the goal is to improve the learning of 
linear classifiers by means of meta-learning a 
distribution over priors. 

In this experiment, each task is a binary classification task, 
which has a task parameter $\w^*$ 
and given an input $\x \sim {\U([-1, 1]^d)}$ 
outputs $y = \ind({\w^*}^\top \x \le 0)$.
The task environment is the set of vectors 
$\w^*\in\R^{d}$ with task distribution 
$\tau = \N(\w^*|\mu_\T, \sigma^2_\T\cdot\text{Id})$
for $\mu_\T=10 \cdot \mathbf{1}$ 
and $\sigma_\T=3$.
The model set consists of linear classifiers, 
$\F=\{\ind(\w^\top \x \le 0): \w \in \R^d\}$, 
and the priors and posteriors are Gaussian 
distributions over their weight vectors. 
Specifically, the priors have the form 
$\N(\w | \mu_P, \sigma_P^2\cdot\text{Id})$ with $\sigma_P=10$, 
from which the posteriors are learned by minimizing 
a PAC-Bayes bound with the logistic regression loss.
The meta-learner learns a Gaussian hyper-posterior 
over the mean of the priors ($\mu_P)$,  based on the hyper-prior $\p(\mu_P) = \N(\mu_P | \mathbf{0}, \sigma_{\p}^2\cdot\text{Id})$ with $\sigma_{\p}=20$.
For background information on the experimental setting, please see the original reference~\citet{rothfuss2022pacoh}.

In Figure~\ref{fig:num_comparison} we plot the numeric values 
of the right-hand side of our bound~\eqref{eq:bound_2},
\ie empirical loss plus the complexity terms, in the 
setting of $d=2$, $m=5$, and different values of $n$. 
We also plot the corresponding values for the bounds from \citet{pentina2014pac, amit2018meta, rothfuss2021pacoh, rezazadeh2022unified, guan2022fast}, as well as the actual quantities
of interest, the meta-test loss, and the meta-training loss.
One can see that our bound has the smallest value, \ie is 
the tightest. 
It becomes non-vacuous already for $n=10$ tasks, while 
the other bounds are still vacuous for $n < 20$.

We would like to emphasize that this analysis 
compares the tightness of the bounds just in the 
setting of prior-based meta-learning, because this 
is the only setting in which previous works can 
provide a guarantee.
The main advantage of our result, however, is its 
applicability to many more settings, where a numeric
comparison is not possible, because previous works 
are not applicable. 

\begin{figure}[t]
\centering
\includegraphics[width=\linewidth]{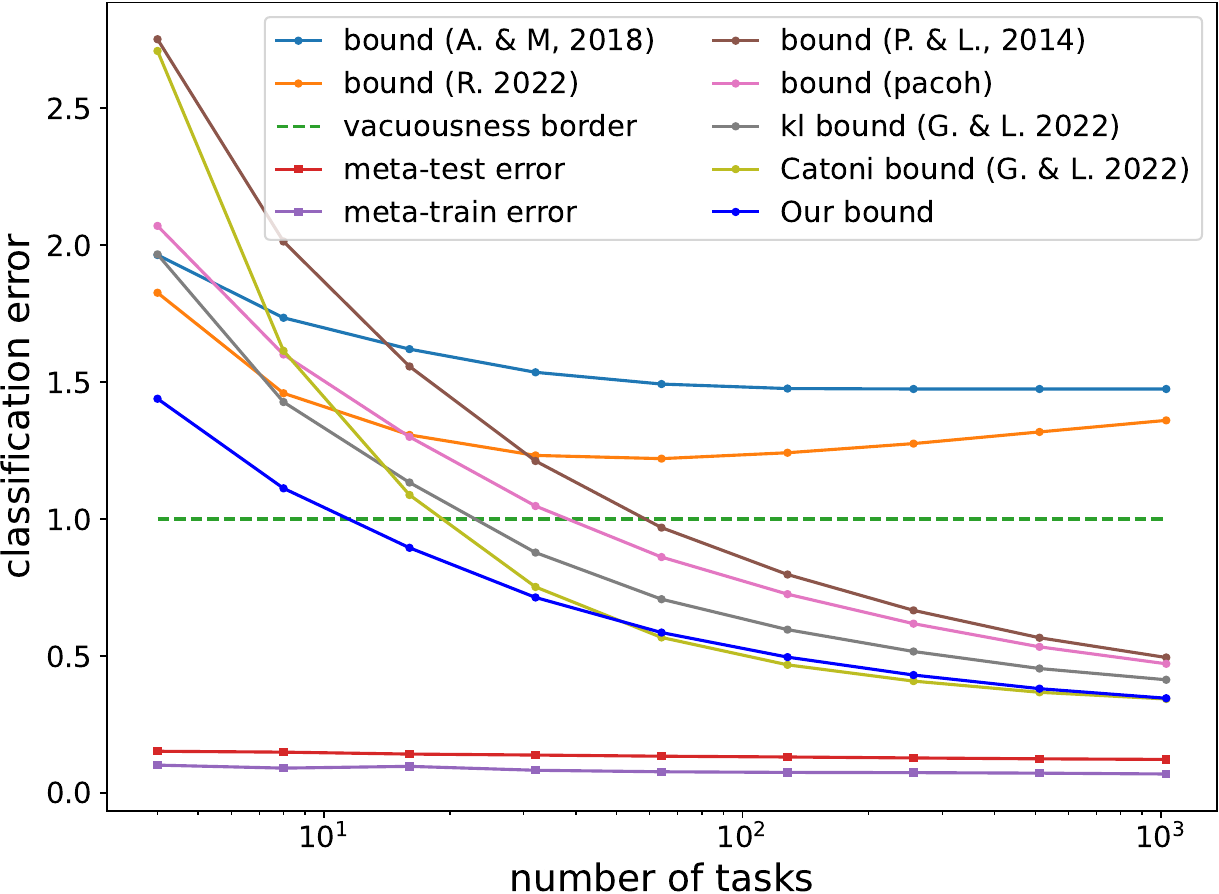}
\caption{Numeric values of different meta-learning bounds (empirical loss plus complexity terms) for the binary classification task described in Section~\ref{sec:numerical_comparison}. 
Values below $1$ are called \emph{non-vacuous}.}
\label{fig:num_comparison}
\end{figure}

\subsection{Learning Initialization as well as Regularization}

In Section~\ref{sec:discussion} we argued that hybrid meta-learning scenarios can be beneficial, for example, learning the initialization of a network as well as a regularization term
for its optimization. %...
As empirical evidence for this setting, we report 
on the standard experiments in the literature 
suggested by \citet{amit2018meta},
in which a stochastic neural network is learned for two exemplary meta-learning scenarios, 
\emph{shuffled pixel} and \emph{permuted labels}.
Both are image classification tasks based on the 
MNIST dataset~\citep{mnist}. 
In the former, tasks differ from each other by 
different permutations of the input pixels at 
a fixed subset of locations. 
In the latter, task differ from each other by 
different permutations of the label space. 
In these experiments, there are 10 training tasks with 600 samples per task. We evaluate the methods on 20 tasks with 100 samples per task. 
For experimental details 
see Appendix~\ref{app:experiments}.

Previous works used this setting as a benchmark for 
PAC-Bayesian meta-learning bounds in the following way: 
the learning rule, $Q(S, P)$, consists of first initializing 
a stochastic network at the mean of the prior, $P$, and 
then training the network by minimizing the right-hand side 
of a PAC-Bayes bound using prior $P$, where the $\KL$-divergence between the prior and the learned model acts as a regularizer 
towards the prior mean.
This description shows that the prior is used in two different ways, for initialization and as regularizer, even though it is not a priori clear why the best choice for these two quantities would be to make them identical. 

In our new framework the roles of initialization and 
regularization can easily be separated, thereby 
allowing us to assess the above question quantitatively. 
To explore this, we use a simple formalization, in 
which each algorithm consists of two distributions 
$(P_0, P_1)$ over models. 
For learning a new task, the stochastic neural network is 
initialized from $P_0$, and then trained by minimizing 
a PAC-Bayes bound with $P_1$ as prior using gradient 
descent.
Formally, for each algorithm, we have $\qa=\pa=\delta_{P_1}$
, and $(\rho_0, \rho_1)$ as meta-posterior, and $(\pi_0, \pi_1)$ as meta-prior, where $\rho_0, \pi_0$ are distributions over $P_0$ and $\rho_1, \pi_1$ are distributions over $P_1$.

For simplicity, we work with Gaussian distributions, 
and we learn their mean and variance by the re-parametrization trick of~\citet{kingma2015variational}.
The optimization is performed as in \citet{amit2018meta}:
one approximately minimizes the meta-learning bound by optimizing for the hyper-posterior, $\rho$ and for separate task posteriors, $Q_1, \dots, Q_n$, for a fixed number of epochs (originally 200). 
For our setup, we use the same implementation, but 
with the difference that we have a meta-posterior 
over two distributions. %
For the first 100 epoch we set them equal, mirroring the 
prior work. Afterwards, however, fix the meta-posterior 
over $P_0$, initialize $Q_1, ..., Q_n$ again by sampling 
from this distribution, and continue the optimization 
for another 100 epochs. 
Note that this results in the same amount of computation as the previous methods, but now split into first learn the network initialization, and then the regularization term conditioned on the learned initialization.

\begin{table}[t]
\centering
\caption{Comparison between our mechanism with prior-based mechanisms.
The separation of initialization and regularization improves the performance for \emph{permuted labels (PL)}, and achieves similar performance for the \emph{shuffled pixels (SP)}.}
\smallskip
\label{table:experiment}
\begin{tabular}{ |c||c|c|   }
\hline
 \multicolumn{3}{|c|}{Meta-Learning PAC Bayes Bounds: Test Error (\%)} \\
 \hline
 Bound & SP & PL \\
 \hline  Independent learning  &  \np{28.9} $\pm$ \np{1.740}  & \np{19.6} $\pm$ \np{1.510} \\
  \cite{amit2018meta} & \textbf{\np{9.92} $\pm$ \np{0.871}} & 13.7 $\pm$ \np{3.460} \\
 \cite{rezazadeh2022unified}  & 11.2 $\pm$ \np{1.040} & 90.1 $\pm$ \np{5.580}\\
\cite{guan2022fast} & 20.5 $\pm$ \np{1.110} & 89.9 $\pm$ \np{0.470} \\
 % \hline
 Ours   & \textbf{\np{9.9} $\pm$ \np{1.140}} & \textbf{\np{7.91} $\pm$ \np{1.700}} \\
 \hline
\end{tabular}
\end{table}

We compare our results with the prior-based bounds of \citet{amit2018meta, rezazadeh2022unified, guan2022fast}, as well as independent learning.
The experimental results are shown in Table~\ref{table:experiment}.
One can see that for the \emph{permuted label} setting, having different parameters reduces the test error. 
This shows that the previous setups were indeed suboptimal, and it thereby confirms the benefits of out framework's added flexibility.
In the \emph{shuffled pixel} setting, the added flexibility did
not yield any benefits, as the system learned almost identical parameters for the initialization and the regularization term. 
Consequently, the results of our framework are essentially 
the same as for previous ones. For more discussion on the results we refer the reader to Appendix \ref{app:experiments}.

\section{Conclusion}
We presented a new framework for the theoretical  
analysis of meta-learning (or learning-to-learn) 
methods. 
Where previous approaches were limited to settings 
that can be formulated as learning a prior distribution 
over models, our new approach takes a more direct
approach and formulates the knowledge transfer as 
learning a preference for learning algorithms.
Our main contributions are two PAC-Bayesian generalization 
bounds that are applicable to essentially all existing 
transfer mechanisms, including \emph{model prototypes}, 
\emph{regularization}, \emph{representation learning}, 
\emph{hypernetworks}, and the transfer of 
\emph{optimization methods} or \emph{hyperparameters}, or combinations thereof. 
We believe our approach will prove useful to put more
practical meta-learning methods onto solid theoretical 
foundation, and ideally to inspire improvement, such
as new forms of regularization, especially for the 
low-data regime.

\section*{Impact Statement}
This paper presents work whose goal is to advance the field of Machine Learning. There are many potential societal consequences of our work, none which we feel must be specifically highlighted here.

\bibliography{ms}

\begin{thebibliography}{40}
\providecommand{\natexlab}[1]{#1}
\providecommand{\url}[1]{\texttt{#1}}
\expandafter\ifx\csname urlstyle\endcsname\relax
  \providecommand{\doi}[1]{doi: #1}\else
  \providecommand{\doi}{doi: \begingroup \urlstyle{rm}\Url}\fi

\bibitem[Alquier(2024)]{alquier2021user}
Alquier, P.
\newblock User-friendly introduction to {PAC-Bayes} bounds.
\newblock \emph{Foundations and Trends in Machine Learning}, 17\penalty0 (2):\penalty0 174--303, 2024.

\bibitem[Amit \& Meir(2018)Amit and Meir]{amit2018meta}
Amit, R. and Meir, R.
\newblock Meta-learning by adjusting priors based on extended {PAC-Bayes} theory.
\newblock In \emph{International Conference on Machine Learning (ICML)}, 2018.

\bibitem[Baxter(2000)]{baxter2000model}
Baxter, J.
\newblock A model of inductive bias learning.
\newblock \emph{Journal of Artificial Intelligence Research (JAIR)}, 12:\penalty0 149--198, 2000.

\bibitem[Chen et~al.(2021)Chen, Shui, and Marchand]{chen2021generalization}
Chen, Q., Shui, C., and Marchand, M.
\newblock Generalization bounds for meta-learning: An information-theoretic analysis.
\newblock In \emph{Conference on Neural Information Processing Systems (NeurIPS)}, 2021.

\bibitem[Denevi et~al.(2019)Denevi, Ciliberto, Grazzi, and Pontil]{denevi2019learning}
Denevi, G., Ciliberto, C., Grazzi, R., and Pontil, M.
\newblock Learning-to-learn stochastic gradient descent with biased regularization.
\newblock In \emph{International Conference on Machine Learning (ICML)}, 2019.

\bibitem[Ding et~al.(2021)Ding, Chen, Levinboim, Goodman, and Soricut]{ding2021bridging}
Ding, N., Chen, X., Levinboim, T., Goodman, S., and Soricut, R.
\newblock Bridging the gap between practice and {PAC-Bayes} theory in few-shot meta-learning.
\newblock In \emph{Conference on Neural Information Processing Systems (NeurIPS)}, 2021.

\bibitem[Farid \& Majumdar(2021)Farid and Majumdar]{farid2021generalization}
Farid, A. and Majumdar, A.
\newblock Generalization bounds for meta-learning via {PAC-Bayes} and uniform stability.
\newblock \emph{Conference on Neural Information Processing Systems (NeurIPS)}, 2021.

\bibitem[Finn et~al.(2017)Finn, Abbeel, and Levine]{finn2017model}
Finn, C., Abbeel, P., and Levine, S.
\newblock Model-agnostic meta-learning for fast adaptation of deep networks.
\newblock In \emph{International Conference on Machine Learning (ICML)}, 2017.

\bibitem[Friedman \& Meir(2023)Friedman and Meir]{friedman2023adaptive}
Friedman, L. and Meir, R.
\newblock Adaptive meta-learning via data-dependent {PAC-Bayes} bounds.
\newblock In \emph{Conference on Lifelong Learning Agents (CoLLAs)}, 2023.

\bibitem[Glorot \& Bengio(2010)Glorot and Bengio]{glorot}
Glorot, X. and Bengio, Y.
\newblock Understanding the difficulty of training deep feedforward neural networks.
\newblock In \emph{International Conference on Artificial Intelligence and Statistics (AISTATS)}, 2010.

\bibitem[Guan \& Lu(2022)Guan and Lu]{guan2022fast}
Guan, J. and Lu, Z.
\newblock Fast-rate {PAC-Bayesian} generalization bounds for meta-learning.
\newblock In \emph{International Conference on Machine Learning (ICML)}, 2022.

\bibitem[Guedj(2019)]{guedj2019primer}
Guedj, B.
\newblock A primer on {PAC-Bayesian} learning.
\newblock \emph{arXiv preprint arXiv:1901.05353}, 2019.

\bibitem[Hellstr{\"o}m \& Durisi(2022)Hellstr{\"o}m and Durisi]{hellstrom2022evaluated}
Hellstr{\"o}m, F. and Durisi, G.
\newblock Evaluated {CMI} bounds for meta learning: Tightness and expressiveness.
\newblock In \emph{Conference on Neural Information Processing Systems (NeurIPS)}, 2022.

\bibitem[Hellstr{\"o}m et~al.(2023)Hellstr{\"o}m, Durisi, Guedj, and Raginsky]{hellstrom2023generalization}
Hellstr{\"o}m, F., Durisi, G., Guedj, B., and Raginsky, M.
\newblock Generalization bounds: Perspectives from information theory and {PAC-Bayes}.
\newblock \emph{arXiv preprint arXiv:2309.04381}, 2023.

\bibitem[Hochreiter et~al.(2001)Hochreiter, Younger, and Conwell]{hochreiter2001learning}
Hochreiter, S., Younger, A.~S., and Conwell, P.~R.
\newblock Learning to learn using gradient descent.
\newblock In \emph{International Conference on Artificial Neural Networks (ICANN)}, 2001.

\bibitem[Kingma et~al.(2015)Kingma, Salimans, and Welling]{kingma2015variational}
Kingma, D.~P., Salimans, T., and Welling, M.
\newblock Variational dropout and the local reparameterization trick.
\newblock In \emph{Conference on Neural Information Processing Systems (NeurIPS)}, 2015.

\bibitem[LeCun \& Cortes(1998)LeCun and Cortes]{mnist}
LeCun, Y. and Cortes, C.
\newblock {MNIST} handwritten digit database.
\newblock http://yann.lecun.com/exdb/mnist/, 1998.

\bibitem[Lee et~al.(2019)Lee, Maji, Ravichandran, and Soatto]{lee2019meta}
Lee, K., Maji, S., Ravichandran, A., and Soatto, S.
\newblock Meta-learning with differentiable convex optimization.
\newblock In \emph{IEEE/CVF Conference on Computer Vision and Pattern Recognition (CVPR)}, 2019.

\bibitem[Li et~al.(2017)Li, Zhou, Chen, and Li]{li2017meta}
Li, Z., Zhou, F., Chen, F., and Li, H.
\newblock Meta-{SGD}: Learning to learn quickly for few-shot learning.
\newblock \emph{arXiv preprint arXiv:1707.09835}, 2017.

\bibitem[Liu et~al.(2021)Liu, Lu, Yan, and Zhang]{liu2021pac}
Liu, T., Lu, J., Yan, Z., and Zhang, G.
\newblock {PAC-Bayes} bounds for meta-learning with data-dependent prior.
\newblock \emph{arXiv preprint arXiv:2102.03748}, 2021.

\bibitem[Maurer(2004)]{maurer2004note}
Maurer, A.
\newblock A note on the {PAC} {Bayesian} theorem.
\newblock \emph{arXiv preprint arXiv:cs.LG/0411099}, 2004.

\bibitem[Maurer(2009)]{maurer2009transfer}
Maurer, A.
\newblock Transfer bounds for linear feature learning.
\newblock \emph{Machine Learning}, 75\penalty0 (3):\penalty0 327--350, 2009.

\bibitem[Maurer et~al.(2016)Maurer, Pontil, and Romera-Paredes]{maurer2016benefit}
Maurer, A., Pontil, M., and Romera-Paredes, B.
\newblock The benefit of multitask representation learning.
\newblock \emph{Journal of Machine Learning Research (JMLR)}, 17\penalty0 (81):\penalty0 1--32, 2016.

\bibitem[McAllester(1998)]{mcallester1998some}
McAllester, D.~A.
\newblock Some {PAC-Bayesian} theorems.
\newblock In \emph{Conference on Computational Learning Theory (COLT)}, 1998.

\bibitem[Nguyen et~al.(2022)Nguyen, Do, and Carneiro]{nguyen2022pac}
Nguyen, C., Do, T.-T., and Carneiro, G.
\newblock {PAC-Bayes} meta-learning with implicit task-specific posteriors.
\newblock \emph{IEEE Transactions on Pattern Analysis and Machine Intelligence (TPAMI)}, 45\penalty0 (1):\penalty0 841--851, 2022.

\bibitem[Nichol et~al.(2018)Nichol, Achiam, and Schulman]{nichol2018first}
Nichol, A., Achiam, J., and Schulman, J.
\newblock On first-order meta-learning algorithms.
\newblock \emph{arXiv preprint arXiv:1803.02999}, 2018.

\bibitem[Pentina \& Lampert(2014)Pentina and Lampert]{pentina2014pac}
Pentina, A. and Lampert, C.~H.
\newblock A {PAC-Bayesian} bound for lifelong learning.
\newblock In \emph{International Conference on Machine Learning (ICML)}, 2014.

\bibitem[Pentina \& Lampert(2015)Pentina and Lampert]{pentina2015lifelong}
Pentina, A. and Lampert, C.~H.
\newblock Lifelong learning with non-iid tasks.
\newblock In \emph{Conference on Neural Information Processing Systems (NeurIPS)}, 2015.

\bibitem[P{\'e}rez-Ortiz et~al.(2021)P{\'e}rez-Ortiz, Rivasplata, Shawe-Taylor, and Szepesv{\'a}ri]{perez2021tighter}
P{\'e}rez-Ortiz, M., Rivasplata, O., Shawe-Taylor, J., and Szepesv{\'a}ri, C.
\newblock Tighter risk certificates for neural networks.
\newblock \emph{The Journal of Machine Learning Research}, 22\penalty0 (1):\penalty0 10326--10365, 2021.

\bibitem[Ravi \& Larochelle(2017)Ravi and Larochelle]{ravi2016optimization}
Ravi, S. and Larochelle, H.
\newblock Optimization as a model for few-shot learning.
\newblock In \emph{International Conference on Learning Representations (ICLR)}, 2017.

\bibitem[Rezazadeh(2022)]{rezazadeh2022unified}
Rezazadeh, A.
\newblock A unified view on {PAC-Bayes} bounds for meta-learning.
\newblock In \emph{International Conference on Machine Learning (ICML)}, 2022.

\bibitem[Riou et~al.(2023)Riou, Alquier, and Ch{\'e}rief-Abdellatif]{riou2023bayes}
Riou, C., Alquier, P., and Ch{\'e}rief-Abdellatif, B.-E.
\newblock {Bayes meets Bernstein at the meta level: an analysis of fast rates in meta-learning with PAC-Bayes}.
\newblock \emph{arXiv preprint arXiv:2302.11709}, 2023.

\bibitem[Rothfuss et~al.(2021)Rothfuss, Fortuin, Josifoski, and Krause]{rothfuss2021pacoh}
Rothfuss, J., Fortuin, V., Josifoski, M., and Krause, A.
\newblock {PACOH}: {Bayes-optimal} meta-learning with {PAC-guarantees}.
\newblock In \emph{International Conference on Machine Learning (ICML)}, 2021.

\bibitem[Rothfuss et~al.(2023)Rothfuss, Josifoski, Fortuin, and Krause]{rothfuss2022pacoh}
Rothfuss, J., Josifoski, M., Fortuin, V., and Krause, A.
\newblock {Scalable PAC-Bayesian Meta-Learning via the PAC-Optimal Hyper-Posterior: From Theory to Practice}.
\newblock \emph{Journal of Machine Learning Research (JMLR)}, 2023.

\bibitem[Schmidhuber(1987)]{schmidhuber1987evolutionary}
Schmidhuber, J.
\newblock Evolutionary principles in self-referential learning, 1987.

\bibitem[Scott et~al.(2024)Scott, Zakerinia, and Lampert]{scott2023pefll}
Scott, J., Zakerinia, H., and Lampert, C.~H.
\newblock {PeFLL: Personalized Federated Learning by Learning to Learn}.
\newblock In \emph{International Conference on Learning Representations (ICLR)}, 2024.

\bibitem[Seldin et~al.(2012)Seldin, Laviolette, Cesa-Bianchi, Shawe-Taylor, and Auer]{seldin2012pac}
Seldin, Y., Laviolette, F., Cesa-Bianchi, N., Shawe-Taylor, J., and Auer, P.
\newblock {PAC-Bayesian} inequalities for martingales.
\newblock \emph{IEEE Transactions on Information Theory}, 58\penalty0 (12):\penalty0 7086--7093, 2012.

\bibitem[Thrun \& Pratt(1998)Thrun and Pratt]{thrun1998}
Thrun, S. and Pratt, L. (eds.).
\newblock \emph{Learning to Learn}.
\newblock Kluwer Academic Press, 1998.

\bibitem[Tian \& Yu(2023)Tian and Yu]{tian2023can}
Tian, P. and Yu, H.
\newblock Can we improve meta-learning model in few-shot learning by aligning data distributions?
\newblock \emph{Knowledge-Based Systems}, 277:\penalty0 110800, 2023.

\bibitem[Zhao et~al.(2020)Zhao, Kobayashi, Sacramento, and von Oswald]{zhao2020meta}
Zhao, D., Kobayashi, S., Sacramento, J., and von Oswald, J.
\newblock Meta-learning via hypernetworks.
\newblock In \emph{NeurIPS Workshop on Meta-Learning (MetaLearn)}, 2020.

\end{thebibliography}
\bibliographystyle{icml2024}

%%%%%%%%%%%%%%%%%%%%%%%%%%%%%%%%%%%%%%%%%%%%%%%%%%%%%%%%%%%%%%%%%%%%%%%%%%%%%%%
%%%%%%%%%%%%%%%%%%%%%%%%%%%%%%%%%%%%%%%%%%%%%%%%%%%%%%%%%%%%%%%%%%%%%%%%%%%%%%%
% APPENDIX
%%%%%%%%%%%%%%%%%%%%%%%%%%%%%%%%%%%%%%%%%%%%%%%%%%%%%%%%%%%%%%%%%%%%%%%%%%%%%%%
%%%%%%%%%%%%%%%%%%%%%%%%%%%%%%%%%%%%%%%%%%%%%%%%%%%%%%%%%%%%%%%%%%%%%%%%%%%%%%%
\newpage
\appendix
\onecolumn

\section{Proofs} \label{app:proofs}

In this section, we provide the proofs of the results in the main body of the paper. 

\subsection{Proof of Theorem \ref{theorem:Main_pa}.}
For the convenience of the reader we restate Theorem \ref{theorem:Main_pa} here and 
then prove it.

\begin{theorem}\label{theorem:Main_pa-appendix}
For any fixed meta-prior $\pi\in\M(\A)$, fixed hyper-prior mapping 
$\p:\A\to\M(\M(\F))$ and any $\delta>0$, it holds with 
probability at least $1-\delta$ over the sampling of the 
training tasks, that for all meta-posterior distributions 
$\rho \in\M(\A)$ over algorithms, and for all hyper-posterior 
mappings $\q:\A\to\M(\M(\F))$ %: \A \rightarrow \M(\M(\F))$
it holds
\begin{align}
    &\er(\rho) \leq \her(\rho) +\sqrt{ \frac{\KL(\rho\|\pi) + \log(\frac{4\sqrt{n}}{\delta})}{2n}}  
+ \sqrt{\frac{\KL(\rho||\pi) + \E_{A \sim \rho}[C_1(A, \q, \p)]+ \log(\frac{8mn}{\delta}) + 1}{2mn}},
\label{eq:bound_1_appendix}
\end{align}
with 
\begin{align}
C_1(A, \q, \p)=&\KL(\qa \| \pa) + \E_{P \sim \qa}  \sum_{i=1}^{n} \KL(A(S_i) || P).
\label{eq:complexity_1-appendix}
\end{align}
\end{theorem}

The beginning of the proof coincides with the steps of the sketch in Section~\ref{sec:proofsketch},
while the later part provides additional details. 
As a reminder, we repeat the definitions of our main objects of interest: the risk of a meta-posterior, $\rho\in\M(\A)$, 
\begin{align}
    \er(\rho) &= \E_{A \sim \rho} \E_{(S,D)\sim\tau}\E_{z \sim D} \ell (z, A(S)),
    \intertext{its empirical analog,}
    \her(\rho) &= \E_{A \sim \rho} \frac{1}{n} \sum_{i=1}^{n}\frac{1}{|S_i|}\sum_{z\in S_i} \ell (z,A(S_i))
\intertext{as well as the intermediate objective, which represents the true risk of the training tasks:}
    \ter(\rho) &= \E_{A \sim \rho} \frac{1}{n} \sum_{i=1}^{n}\E_{z \sim D_i} \ell (z, A(S_i)).
\end{align}

We divide the proof into two parts. First, we bound the difference of the true risks between training tasks and future tasks $\er(\rho) - \ter(\rho)$. For the second part, we bound the difference between the true risk and empirical risk of training tasks $\ter(\rho) - \her(\rho)$, and by combining the two bounds we obtain the final result.

\paragraph{Part I} For the first part we can use classical PAC-Bayes arguments, because 
 $\er(\rho)$ and $\ter(\rho)$ differ only in the fact that one is an empirical average
 with respect to the tasks while the other it is expectation. 
 Consequently, one obtains:

\begin{lemma}\label{lemma:part1-appendix}
For all $\delta>0$ it holds with probability at least $1 - \frac{\delta}{2}$ over
the sampling of tasks that for all meta-posteriors $\rho \in \M(\A)$:
\begin{align}
    &\er(\rho) - \ter(\rho) \le \sqrt{ \frac{\KL(\rho||\pi) + \log(\frac{4\sqrt{n}}{\delta})}{2n}}.
\end{align}
\end{lemma}

\begin{proof}
For each algorithm $A$ and task $T=(D, S)$ we define the loss as $l_\text{task}(A, T) = \E_{z \sim D} \ell (z, A(S))$. For this loss, $\ter(A)$ is the empirical risk, and $\er(A)$ is the true risk. 
Applying the standard PAC-Bayes bounds \citep{maurer2004note, perez2021tighter} 
to this setting results in this lemma.
\end{proof}

\paragraph{Part II}
We define the following two functions that produce distributions over 
$\A\times\M(\F)\times \F^{\otimes n}$, \ie they assigns joint probabilities 
to tuples, $(A, P, f_1, ..., f_n)$, which contain a algorithm, 
a prior over models, and $n$ models.
%:

\begin{itemize}
    \item \emph{Posterior $\posterior$:} given as input a meta-posterior $\rho$ over algorithms
    and a hyper-posterior mapping $\q$ as input, it outputs the distribution over 
    $\A\times\M(\F)\times \F^{\otimes n}$ with the following generating process:
    \emph{i)} sample an algorithm $A\sim\rho$, \emph{ii)} sample a prior $P\sim\qa$, 
    \emph{iii)} for each task, $i=1,\dots,n$, sample a model $f_i\sim A(S_i)$. 
\item \emph{Prior $\prior$:} given as input a meta-prior $\pi$ over algorithms and a hyper-prior mapping $\p$ as input, it outputs the distribution over $\A\times\M(\F)\times \F^{\otimes n}$ with the following generating process: \emph{i)} sample an algorithm $A\sim \pi$, 
\emph{ii)} sample a prior $P\sim\pa$, \emph{iii)} for each task, $i=1,\dots,n$, sample 
a model $f_i\sim P$.
\end{itemize}

Note that the inputs to $\posterior$ are data-dependent and will be learned using the data.
In contrast, the input to $\prior$ are data-independent and need to be fixed before seeing the data.
With these definitions, we state the following key lemma:

\begin{lemma}\label{lemma1:part2-appendix}
For any fixed meta-prior $\pi\in\M(\A)$, fixed hyper-prior mapping 
$\p:\A\to\M(\M(\F))$, any $\delta>0$, and any $\lambda>0$, it holds with 
probability at least $1 - \frac{\delta}{2}$ over the 
sampling of the training datasets that for all 
meta-posteriors $\rho \in\M(\A)$ over algorithms, and 
for all hyper-posterior functions $\q: \A \rightarrow \M(\M(\F))$:
\begin{align}
         \ter(\rho) - \her(\rho)
          &\leq \frac{1}{\lambda} \KL\big(\posterior \| \prior\big) + \frac{1}{\lambda} \log(\frac{2}{\delta}) +  \frac{\lambda}{8nm}.
            \label{eq:lemma1:part2-appendix1}
\end{align}
\end{lemma}

\begin{proof}
First for any task $i$ and any model $f_i$ we define:
\begin{align}
    \Delta_i(f_i) = \E_{z \sim D_i} \ell (z,f_i) - \frac{1}{|S_i|} \sum_{z\in S_i} \ell (z,f_i).
\end{align} 
By this definition and the definitions of $\ter$ and $\her$ we have:
\begin{align}
    \E_{(A, P, f_1, \dots, f_n) \sim \posterior} &\Big[\frac{1}{n} \sum_{i=1}^{n} \Delta_i(f_i) \Big] = \ter(\rho) - \her(\rho) 
\end{align}

Using this equation and the \emph{change of measure inequality}~\citep{seldin2012pac} between the two distributions $\posterior$ and $\prior$, for any $\lambda>0$, any $\rho\in\M(\A)$ and any $\q$, we have:
\begin{align}
\begin{split}
        &\ter(\rho) - \her(\rho) - \frac{1}{\lambda}  \KL(\posterior || \prior)
        \le \frac{1}{\lambda} \log \E_{(A, P, f_1, f_2, ..., f_n) \sim \prior} \prod_{i=1}^{n} e^{\frac{\lambda}{n} \Delta_i(f_i)}
        \label{eq:change_of_measure_app1}
\end{split}
\end{align}

where the second inequality is due to the \emph{change of measure inequality}~\citep{seldin2012pac}.

By the construction of $\prior$, we have
\begin{align}    
\E_{S_1, \dots, S_n} \E_{(A, P, f_1, f_2, ..., f_n) \sim \prior} \prod_{i=1}^{n} e^{\frac{\lambda}{n} \Delta_i(f_i)} = &\E_{S_1, \dots, S_n} \E_{A \sim \pi} \E_{P \sim \pa} \E_{f_1 \sim P} \dots \E_{f_n \sim P} \prod_{i=1}^{n} e^{\frac{\lambda}{n} \Delta_i(f_i)},
\intertext{and, because it is independent of $S_1,\dots,S_n$, we can rewrite this as}
&= \E_{A \sim \pi} \E_{P \sim \pa} \E_{S_1} \E_{f_1 \sim P} e^{\frac{\lambda}{n} \Delta_1(f_1)} \dots \E_{S_n}\E_{f_n \sim P} e^{\frac{\lambda}{n} \Delta_n(f_n)}.
\\
&= \E_{A \sim \pi} \E_{P \sim \pa} 
\prod_{i=1}^n \E_{S_i} \E_{f_i \sim P} e^{\frac{\lambda}{n} \Delta_i(f_i)}
\label{eq:reordering}
\end{align}
Each $\Delta_i(f_i)$ is a bounded random variable with support in 
an interval of size $1$. By Hoeffding's lemma we have 
\begin{align}
    \E_{S_i}\E_{f_i \sim P} e^{\frac{\lambda}{n} \Delta_i(f_i)} \le e^{\frac{\lambda^2}{8n^2m}}.
    \label{eq:hoeffding1}
\end{align}
Therefore, by combining \eqref{eq:reordering} and \eqref{eq:hoeffding1} we have:
\begin{align}    
\E_{S_1, \dots, S_n} \E_{(A, P, f_1, f_2, ..., f_n) \sim \prior} \prod_{i=1}^{n} e^{\frac{\lambda}{n} \Delta_i(f_i)} \le e^{\frac{\lambda^2}{8nm}}.
\end{align}
By Markov's inequality, for any $\epsilon>0$ we have
\begin{align}
    \mathbb{P}_{S_1, \dots, S_n}\Big(\E_{(A, P, f_1, f_2, ..., f_n) \sim \prior} \prod_{i=1}^{n} e^{\frac{\lambda}{n} \Delta_i(f_i)} \ge e^\epsilon\Big) \le e^{\frac{\lambda^2}{8nm} - \epsilon} \label{eq:markov1}
\end{align}
Hence by combining \eqref{eq:change_of_measure_app1} and  \eqref{eq:markov1} we get for any $\epsilon$:
\begin{align}
        &\mathbb{P}_{S_1, \dots, S_n}\Big(\exists{\rho, \q}: \ter(\rho) - \her(\rho) - \frac{1}{\lambda}  \KL(\posterior || \prior) \ge \frac{1}{\lambda} \epsilon\Big) \le e^{\frac{\lambda^2}{8nm} - \epsilon},
\end{align}
or, equivalently, it holds for any $\delta>0$ with probability at least $1 - \frac{\delta}{2}$:
\begin{align}
        \forall{\rho, \q}: \ter(\rho) - \her(\rho) \le \frac{1}{\lambda}  \KL(\posterior || \prior) + \frac{1}{\lambda} \log(\frac{2}{\delta}) +  \frac{\lambda}{8nm}.
        \label{eq:bigKL-appendix}
\end{align}
\end{proof}

The following lemma splits the $\KL$ term of \eqref{eq:bigKL-appendix} into more interpretable quantities.
\begin{lemma}
    For the posterior, $\posterior$, and prior, $\prior$, defined above it holds:
    \begin{align}
        & \KL(\posterior || \prior) =  \KL(\rho \| \pi)  
         + \E_{A \sim \rho} \Big[\KL(\qa\| \pa) + \E_{P \sim \qa}  \sum_{i=1}^{n} \KL(A(S_i) || P)\Big].
    \end{align}
    \label{lemma:KL_app}
\end{lemma}
\begin{proof}
    \begin{align}
        \KL(\posterior || \prior) &=  \E_{A \sim \rho} \Bigg[\E_{P \sim \qa} \Big[ \E_{f_i \sim A(S_i)} \ln \frac{\rho(A) \qa(P) \prod_{i=1}^n A(S_i)(f_i)}{\pi(A) \pa(P) \prod_{i=1}^n P(f_i)} \Big] \Bigg]
            \\& = \E_{A \sim \rho} \Big[\ln \frac{\rho(A)}{\pi(A)}\Big] +  \E_{A \sim \rho} \Bigg[\E_{P \sim \qa} \Big[\ln \frac{\qa(P)}{\pa(P)} \Big] + \E_{P \sim \qa} \Big[\sum_{i=1}^n \E_{f_i \sim A(S_i)} \ln \frac{A(S_i)(f_i)}{P(f_i)} \Big] \Bigg]
            \\& = \KL(\rho \| \pi)  
         + \E_{A \sim \rho} \Bigg[\KL(\qa\| \pa) + \E_{P \sim \qa}  \sum_{i=1}^{n} \KL(A(S_i) || P)\Bigg].
    \end{align}
\end{proof}

\paragraph{Part III} We now combine the above parts to prove Theorem \ref{theorem:Main_pa}.

\begin{proof}
To get tight guarantees, we need to choose the value of $\lambda$ 
in Lemma~\ref{lemma1:part2-appendix} an optimal data-dependent way, 
but the statement of the Lemma holds only for individual values 
of $\lambda$.
Therefore, we first create an version of inequality~\eqref{eq:lemma1:part2-appendix1} 
by instantiating it for each $\lambda \in \Lambda$ with $\Lambda=\{1, \dots, 4mn\}$,
and then applying a union-bound.
It follows that 
\begin{align}
        \mathbb{P}_{S_1, \dots, S_n}&\Big(\forall{\rho, \q, \lambda \in \Lambda}: \ter(\rho) - \her(\rho) \le \frac{1}{\lambda}  \Big[\KL\big(\posterior || \prior\big) + \frac{1}{\lambda} \log(\frac{8mn}{\delta})\Big] +  \frac{\lambda}{8nm}\Big) \ge 1 - \frac{\delta}{2}
\end{align}
Note that for real-valued $\lambda>1$, it holds that $\lfloor \lambda \rfloor \le \lambda$ and $\frac{1}{\lfloor \lambda \rfloor} \le \frac{1}{\lambda - 1}$. %
Thereby, we can allow real-valued $\lambda$ and obtain
\begin{align}
        \mathbb{P}_{S_1, \dots, S_n}&\Big(\forall{\rho, \q, \lambda \in (1, 4mn]}: \ter(\rho) - \her(\rho)
         \leq \underbrace{\frac{1}{\lambda - 1}  \Big[\KL\big(\posterior || \prior\big) + \log(\frac{8mn}{\delta}) \Big] +  \frac{\lambda}{8mn}}_{=:\Gamma(\lambda)}\Big) \ge 1 - \frac{\delta}{2}\label{eq:union_bound}
\end{align}
For any choice of $\rho, \q$, let $\lambda^*=
\sqrt{8mn (\KL(\posterior || \prior) + \log(\frac{8mn}{\delta}))+1}$.
If $\lambda^*>4mn$, that implies 
\begin{align}
    \sqrt{\frac{\KL(\posterior || \prior) + \log(\frac{8mn}{\delta}) + 1}{2mn}} > 1.
\end{align} 
% %
Otherwise, $\lambda^*\in(1,4mn]$, so inequality~\eqref{eq:union_bound} holds, and we have 
\begin{align}
    \Gamma(\lambda^*) < \sqrt{\frac{\KL(\posterior || \prior) + \log(\frac{8mn}{\delta}) + 1}{2mn}}.
\end{align} 
Therefore,
\begin{align}
        &\mathbb{P}_{S_1, \dots, S_n}\Big(\forall{\rho, \q}:  \ter(\rho) - \her(\rho)
        \le \sqrt{\frac{\KL(\posterior || \prior) + \log(\frac{8mn}{\delta}) + 1}{2mn}}\Big) \ge 1 - \frac{\delta}{2}.
\end{align}
In combination with Lemma \ref{lemma:KL_app}, 
with probability at least $1 - \frac{\delta}{2}$ it holds for all $\rho\in\M(\A), \q:\A\to\M(\F)$:
\begin{align}
        &\ter(\rho) - \her(\rho)
        \le  
         \sqrt{\frac{\KL(\rho||\pi) + \E_{A \sim \rho}[C_1(A, \q, \p)]+ \log(\frac{8mn}{\delta}) + 1}{2mn}}.
        \label{eq:ter-her}
\end{align}  
where $C_1(A, \q, \p)$ is defined as in \eqref{eq:complexity_1-appendix}.
Combining \eqref{eq:ter-her} and Lemma~\ref{lemma:part1-appendix} by a union bound concludes the proof. 
\end{proof}

\subsection{Proof of Theorem \ref{theorem:Main_p}.}
We now restate and prove Theorem~\ref{theorem:Main_p}.

\begin{theorem}\label{theorem:Main_p-appendix}
For any fixed meta-prior $\pi\in\M(\A)$, any fixed hyper-prior 
$\p\in\M(\M(\F))$ and any $\delta>0$, it holds with probability 
at least $1-\delta$ over the sampling of the datasets, 
that for all meta-posterior distributions $\rho \in\M(\A)$ over algorithms,
and for all hyper-posterior functions $\q: \A \rightarrow \M(\M(\F))$
we have
\begin{align}
    \er(\rho) &\leq \her(\rho) + \sqrt{ \frac{\KL(\rho||\pi) + \log(\frac{4\sqrt{n}}{\delta})}{2n}}  
    +  \ \E_{A \sim \rho} \sqrt{\frac{C_2(A, \q, \p)+ \log(\frac{8mn}{\delta}) + 1}{2mn}},
\label{eq:bound_2-appendix}
\intertext{with}
C_2(A, \q, \p) &= \KL(\qa \| \p)  + \E_{P \sim \qa}  \sum_{i=1}^{n} \KL(A(S_i) || P).
\label{eq:complexity_2-appendix}
\end{align}
\end{theorem}

The three-step proof largely follows that of Theorem~\ref{theorem:Main_pa}, 
except for some differences that emerge because the constant hyper-prior allows 
some arguments to hold (with high probability over the datasets) uniformly for all algorithms at the same time.

\paragraph{Part I} We bound the different $\er(\rho)$ and $\ter(\rho)$ the same way as in Lemma~\ref{lemma:part1-appendix}.

\paragraph{Part II} 
We define the following two functions that produce distributions over $\M(\F)\times \F^{\otimes n}$, \ie they assign joint probabilities to tuples, $(P, f_1, ..., f_n)$ that contain a prior over models and $n$ models.

%\smallskip
\begin{itemize}
    \item \emph{Posterior $\post$:} given as input an algorithm $A\in\A$ and 
    a hyper-posterior mapping $\q:\A\to\M(\F)$ as input, it outputs the 
    distribution over $\M(\F)\times\F^{\otimes n}$ with the following generating 
    process:
    \emph{i)} sample a prior $P\sim\qa$, 
    \emph{ii)} for each task, $i=1,\dots,n$, sample a model $f_i\sim A(S_i)$. 
\item \emph{Prior $\pri$:} given a hyper-prior $\p\in\M(\F)$ as input, it outputs the distribution over $\M(\F)\times \F^{\otimes n}$ with the following generating process: \emph{i)} sample a prior $P\sim\pa$, \emph{ii)} for each task, $i=1,\dots,n$, sample a model $f_i\sim P$.
\end{itemize}

\begin{lemma}\label{lemma1:part2-appendix2}
For any fixed meta-prior $\pi\in\M(\A)$, fixed hyper-prior $\p\in\M(\M(\F))$,  any $\delta>0$, and any $\lambda>0$, it holds with probability $1 - \frac{\delta}{2}$ over the sampling of datasets from training tasks that for all algorithms $A \in \A$ and for all hyper-posterior functions $\q: \A \rightarrow \M(\M(\F))$:
\begin{align}
         \ter(A) - \her(A)
         & \le \frac{1}{\lambda} \KL\big(\post \| \pri\big)  +    \frac{1}{\lambda} \log(\frac{2}{\delta}) +  \frac{\lambda}{8nm},
\intertext{where}
    \ter(A) &= \frac{1}{n} \sum_{i=1}^{n}\E_{z \sim D_i} \ell (z, A(S_i)),
\qquad\text{and}\qquad
    \her(A) = \frac{1}{n} \sum_{i=1}^{n} \frac{1}{|S_i|} \sum_{z\in S_i} \ell (z, A(S_i)).
\end{align}
\end{lemma}
\begin{proof}
First for any task $i$ and any model $f_i$ we define:
\begin{align}
    \Delta_i(f_i) = \E_{z \sim D_i} \ell (z, f_i) - \frac{1}{|S_i|} \sum_{z\in S_i} \ell (z, f_i)
\end{align} 
By this definition and the definitions of $\ter$ and $\her$ we have:
\begin{align}
    \E_{(P, f_1, f_2, ..., f_n) \sim \post} \Big[\frac{1}{n} \sum_{i=1}^{n} \Delta_i(f_i) \Big] &= \ter(A) - \her(A) 
\end{align}

Using this equation and the \emph{change of measure inequality}~\citep{seldin2012pac} between the two distributions $\post$ and $\pri$, for any $\lambda>0$, any $A\in\A$ and any $\q:\A\to\M(\M(\F))$, we have:
\begin{align}
        \ter(A) - \her(A) - \frac{1}{\lambda}  \KL\big(\post || \pri\big)
        & \leq \frac{1}{\lambda} \log \E_{(P, f_1, f_2, ..., f_n) \sim \pri} \prod_{i=1}^{n} e^{\frac{\lambda}{n} \Delta_i(f_i)}
        \label{eq:change_of_measure_app2}
\end{align}

Because $\pri$ is independent of $S_1, ..., S_n$, we have
\begin{align}    
\E_{S_1, \dots, S_n} \E_{(P, f_1, f_2, ..., f_n) \sim \prior} \prod_{i=1}^{n} e^{\frac{\lambda}{n} \Delta_i(f_i)} 
&= \E_{S_1, \dots, S_n} \E_{P \sim \p} \E_{f_1 \sim P} \dots \E_{f_n \sim P} \prod_{i=1}^{n} e^{\frac{\lambda}{n} \Delta_i(f_i)} 
\\
&=  \E_{P \sim \p} \E_{S_1} \E_{f_1 \sim P} e^{\frac{\lambda}{n} \Delta_1(f_1)} \dots \E_{S_n}\E_{f_n \sim P} e^{\frac{\lambda}{n} \Delta_n(f_n)}
\\
&=  \E_{P \sim \p}\prod_{i=1}^n \E_{S_i}\E_{f_i \sim P} e^{\frac{\lambda}{n} \Delta_i(f_i)}
\label{eq:reordering2}
\end{align}

Each $\Delta_i(f_i)$ is a bounded random variable with support in 
an interval of size $1$. By Hoeffding's lemma we have 
\begin{align}
    \E_{S_i}\E_{f_i \sim P} e^{\frac{\lambda}{n} \Delta_i(f_i)} \le e^{\frac{\lambda^2}{8n^2m}}.
    \label{eq:hoeffding2}
\end{align}

Therefore by combining \eqref{eq:reordering2} and \eqref{eq:hoeffding2} we have:
\begin{align}    
\E_{S_1, \dots, S_n} \E_{(P, f_1, f_2, ..., f_n) \sim \pri} \prod_{i=1}^{n} e^{\frac{\lambda}{n} \Delta_i(f_i)} \le e^{\frac{\lambda^2}{8nm}}.
\end{align}
By Markov's inequality, for any $\epsilon>0$ we have
\begin{align}
    \mathbb{P}_{S_1, \dots, S_n}\Big(\E_{(P, f_1, f_2, ..., f_n) \sim \pri} \prod_{i=1}^{n} e^{\frac{\lambda}{n} \Delta_i(f_i)} \ge e^\epsilon\Big) \le e^{\frac{\lambda^2}{8nm} - \epsilon} \label{eq:markov2}
\end{align}
Hence by combining \eqref{eq:change_of_measure_app2} and  \eqref{eq:markov2} we get
\begin{align}
        &\mathbb{P}_{S_1, \dots, S_n}\Big(\exists{A, \q}: \ter(A) - \her(A) - \frac{1}{\lambda}  \KL(\post || \pri) \ge \frac{1}{\lambda} \epsilon\Big) \le e^{\frac{\lambda^2}{8nm} - \epsilon},
\end{align}
or, equivalently, it holds for any $\delta>0$ with probability at least $1 - \frac{\delta}{2}$:
\begin{align}
        \forall{A, \q}: \quad\ter(A) - \her(A) \le \frac{1}{\lambda}  \KL(\post || \pri) + \frac{1}{\lambda} \log(\frac{2}{\delta}) +  \frac{\lambda}{8nm}
\label{eq:bigKL-appendix2}
\end{align}
\end{proof}

The following lemma splits the $\KL$ term of \eqref{eq:bigKL-appendix2} into more interpretable quantities.
\begin{lemma}
    For the posterior, $\post$, and prior, $\pri$, defined above it holds:
    \begin{align}
    \begin{split}
        & \KL\big(\post || \pri\big) =  \KL(\qa\| \p) + \E_{P \sim \qa}  \sum_{i=1}^{n} \KL(A(S_i) || P)
    \end{split}
    \end{align}
    \label{lemma:KL_app2}
\end{lemma}
\begin{proof}
\begin{align}
    \KL(\post || \pri) &=  \E_{P \sim \qa} \Big[ \E_{f_i \sim A(S_i)} \ln \frac{\qa(P) \prod_{i=1}^n A(S_i)(f_i)}{\p(P) \prod_{i=1}^n P(f_i)} \Big]
    \\& = \E_{P \sim \qa} \Big[\ln \frac{\qa(P)}{\p(P)} \Big] + \E_{P \sim \qa} \Big[\sum_{i=1}^n \E_{f_i \sim A(S_i)} \ln \frac{A(S_i)(f_i)}{P(f_i)} \Big] \Bigg]
    \\& = \KL(\qa\| \p) + \E_{P \sim \qa}  \sum_{i=1}^{n} \KL(A(S_i) || P)
\end{align}
\end{proof}

\paragraph{Part III} We now finish the proof of Theorem \ref{theorem:Main_p} by combining the above results.

\begin{proof}
By applying a union bound for all the values of $\lambda \in \Lambda$ with $\Lambda=\{1, \dots, 4mn\}$ we obtain that
\begin{align}
        \mathbb{P}_{S_1, \dots, S_n}\Big(\forall{A, \q}:  \ter(A) - \her(A)
        &\le \sqrt{\frac{\KL\big(\post || \pri\big) + \log(\frac{8mn}{\delta}) + 1}{2mn}}\Big) \ge 1 - \frac{\delta}{2}.
\end{align}
Because this inequality holds (with high probability over the datasets) for all algorithms at the same time, it also holds in expectation over algorithms with respect to any distribution. Therefore we have:
\begin{align}
        &\mathbb{P}_{S_1, \dots, S_n}\Big(\forall{\rho, \q}: \E_{A \sim \rho} [\ter(A) - \her(A)]
        \le \E_{A \sim \rho} \sqrt{\frac{\KL(\post || \pri) + \log(\frac{8mn}{\delta}) + 1}{2mn}}\Big) \ge 1 - \frac{\delta}{2},
\end{align}
or, equivalently, 
\begin{align}
        &\mathbb{P}_{S_1, \dots, S_n}\Big(\forall{\rho, \q}: \ter(\rho) - \her(\rho)
        \le \E_{A \sim \rho} \sqrt{\frac{\KL(\post || \pri) + \log(\frac{8mn}{\delta}) + 1}{2mn}}\Big) \ge 1 - \frac{\delta}{2}.
\end{align}
In combination with Lemma \ref{lemma:KL_app2}, with probability at least $1 - \frac{\delta}{2}$ we have for all $\rho\in\M(\A), \q:\A\to\M(\M(\F))$:
\begin{align}
        &\ter(\rho) - \her(\rho)
        \le  
          \E_{A \sim \rho} \sqrt{\frac{C_2(A, \q, \p)+ \log(\frac{8mn}{\delta}) + 1}{2mn}}
        \label{eq:ter-her2}
\end{align}  
where $C_2(A, \q, \p)$ is defined as in \eqref{eq:complexity_2-appendix}.
Combining \eqref{eq:ter-her2} and Lemma \ref{lemma:part1-appendix} concludes the proof. 
\end{proof}

\newpage

\section{Experimental Details}\label{app:experiments}

In this section, we provide the details of our experiments.

\subsection{Setup}
We follow the experimental setup proposed in \citet{amit2018meta} for benchmarking meta-learning methods.  
The experiment consists of two types of tasks based on the MNIST dataset \citep{mnist}.
In the first one, each task is the MNIST classification task with a task-specific permutation in the labels.
In the second experiment, each task has the MNIST images as samples, but a task-specific shuffle is applied to a subset of 200 pixels of the input.
In both cases, there are 10 training tasks and 20 test tasks. We use 600 samples per training task and 100 samples per test task.
This choice corresponds to a setup in which independent learning cannot be expected to 
provide good performance (the number of samples per task is small), and meta-learning 
is necessary.

\subsection{Meta-learning Algorithm}
As discussed in Section~\ref{sec:background}, the meta-learning mechanism introduced in \citet{amit2018meta} 
 is to learn a hyper-posterior over priors in the training phase. 
For a future task, they minimize a PAC-Bayes bound based on this prior. 
As mentioned in Section \ref{sec:experiments}
they train a neural network and use the same prior as the initialization point. This setup does not allow learning an initialization different from the prior used for regularization. 

To show the benefits of the additional freedom provided by our framework, we use the same procedure except that we learn a separate initialization for the network which can differ from the prior used in the objective, as discussed in Sections \ref{sec:discussion} and \ref{sec:experiments}.
In the stochastic setting, we learn a meta-posterior $\rho$ over the initialization prior and the regularization prior in the training phase.
For future tasks, we sample the two distributions $(P_0, P_1)$ from $\rho$, initialize our stochastic neural network by $P_0$, and optimize a PAC-Bayes bound with prior $P_1$.
Note that the meta learner is free to make use of the added flexibility by learning $P_0$ different from $P_1$, or to recover the previous setup by learning $P_0$ identical to $P_1$. 
We do not have to fear overfitting from the larger set of parameters, because the objective is based on a generalization bound that enforces appropriate regularization.

We use Gaussians for all distributions, which allows us to compute the complexity terms in closed form.
More precisely, let $d$ be the number of weights in our neural network and we represent the prior and posteriors by their mean $\mu_i$ and the log variance value $\log \sigma_i$ for the weight $w_i$. 
Formally, we represent each $\rho$ as $\rho_0\times\rho_1$, which $\rho_i$ is a distribution over $P_i$, and has the form of $\mathcal{N}(\theta_i, \kappa_{\rho}^2 I_{2d \times 2d})$.
We use a fixed parameter for $\kappa_\rho$,
and we learn the means $\theta_i$. 
The meta-priors also have the same form, \ie 
$\pi = \pi_0\times\pi_1$ and $\pi_i = \mathcal{N}(0, \kappa_{\pi}^2 I_{2d \times 2d})$, with fixed $\kappa_{\pi}$.

To use our generalization bounds for this mechanism,
we apply Theorem \ref{theorem:Main_pa} (with $\delta = 0.1$) and 
we set we $\q(P_0, P_1)=\p(P_0, P_1)=\delta_{P_1}$.
The result is a bound
\begin{align}
\er(\rho)  &\le \her(\rho) +\sqrt{ \frac{\KL(\rho\|\pi) + \log(\frac{4\sqrt{n}}{\delta})}{2n}}  
+   \sqrt{\frac{\KL(\rho\|\pi) + \E_{(P_0, P_1) \sim \rho} \sum_{i=1}^{n} \KL(A(S_i) || P_1)]+ \log(\frac{8mn}{\delta}) + 1}{2mn}}
\label{eq:64}
\end{align} 
in which $\KL(\rho \| \pi)$ has the following form:
\begin{align}
    \KL(\rho \| \pi) = \frac{4d \kappa_{\rho} + \| \theta_0 \|^2 + \| \theta_1 \|^2}{2\kappa_{\pi}} - 2d + 4d\log(\frac{\kappa_{\pi}}{\kappa_{\rho}})
\end{align}

\textbf{Training phase.} 
In the training phase, we optimize the 
right-hand side of \eqref{eq:64} to find the meta-posterior $\rho$.
As in~\citet{amit2018meta} we use the Monte Carlo method to approximate the values for calculating the expectation terms, and use the re-parametrization trick~\citep{kingma2015variational} to optimize the expected value of the $\KL(A(S_i)||P_1)$ terms.

To find $\rho$ we follow the optimization procedure 
defined in \citet{amit2018meta}. For each task $i$, we assign a stochastic neural network $Q_i$, initiated in the following way: 
The mean of each weight is initiated randomly
with the Glorot method~\citep{glorot}, and the $\log$-variance of each weight is initiated randomly from $\mathcal{N}(-10, 0.1^2)$.

In their original optimization procedure \citet{amit2018meta} optimize their objective 
for 200 epochs to find the hyper-posterior.
Because our meta-distribution has two parts: 
one for initialization and one for regularization,
we add the following change to this procedure: 
In the first 100 epochs, we assume $\rho_0$ and $\rho_1$ are equal and we minimize the bound to find $\rho_0=\rho_1$ and posteriors $Q_i$s. 
After 100 epochs, we fix $\rho_0$, initialize the $Q_i$s by sampling from $\rho_0$ (Since $\rho_0$ is supposed 
to be the meta-distribution over initialization prior)
and optimize the bound for $\rho_1$ and the $Q_i$s for another 100 epochs.

\textbf{Future tasks.} For a future task $T_\text{new}$ with 
its training dataset $S_\text{new}$ we learn its posterior as follows:
we sample the initialization prior from $\rho_0$ 
to initiate a stochastic neural network $Q_\text{new}$.
Then, similar to previous works, we optimize the
following PAC-Bayesian bound with the prior $P_1$ 
sampled from $\rho_1$ for 100 epochs.
\begin{align}
     &\her(Q_\text{new}) + \sqrt{\frac{\KL(Q_\text{new} || P_1)+ \log(\frac{8m}{\delta}) + 1}{2m}}
\end{align}
We use Monte Carlo sampling to approximate the expectations.

\subsection{Implementation and Numeric Results} %
Our implementation is based on the code of \citet{amit2018meta}, 
except that we fixed a bug in their computation of the $\KL$-divergences,
which was also present in later works derived from it. 
Furthermore, we corrected an issue with how the gradients 
of the objective in \citet{rezazadeh2022unified} were computed. 
All experiments were done with the corrected implementation.\footnote{\href{https://github.com/hzakerinia/Flexible-PAC-Bayes-Meta-Learning/}
{\url{https://github.com/hzakerinia/Flexible-PAC-Bayes-Meta-Learning/}}}

We use the same network architectures as proposed~\citet{amit2018meta}: a small ConvNet 
with two convolutional layers and one fully-connected layer for the \emph{permuted labels} task, and a three-layer fully-connected network for the \emph{shuffled pixels} task. 
For further details on the architecture, see the original reference.
We used the Adam optimizer with a learning rate of $10^{-3}$ and the number of Monte Carlo iterations was 1 in all experiments.
For the fixed parameters of the minimization objective,
we put $\delta=0.1$, and for the variances of meta-prior $\pi$ and meta-posterior $\rho$, we set $\kappa_{\pi} = 10^2$ and $\kappa_{\rho} = 10^{-3}$. 
Moreover, the batch size is 128 and the used loss function is Cross-Entropy loss.

The experimental results are shown in the Table~\ref{table:experiment}.
We compare our results with the following prior-based works: 
The (MLAP-M) bound of \citet{amit2018meta},
The Classic bound of \citet{rezazadeh2022unified},
and the kl-bound of \citet{guan2022fast}.
The results confirm that the extra flexibility of our framework can be beneficial.
Specifically, in the \emph{permuted labels} experiment, having a different initialization and regularization helps a lot. In the \emph{shuffled pixel} setting, the flexibility does not help and we get the same performance as the previous methods.
A noteworthy feature of Table~\ref{table:experiment} is the high 
error of \citet{rezazadeh2022unified} and \citet{guan2022fast} in 
the permuted labels task. 
These are a consequence of the fact that the complexity terms in their 
bounds are very big in low-data regime that we are interested in
(where meta-learning is meant to help). 
As a result, the optimization mainly attempt as reducing the complexity terms, 
which leads to underfitting and classification performance as good as a 
random guess ($\approx 90\%$ error).
The same problem occurs for the Catoni-type bound in \citet{guan2022fast} 
for both experimental settings, so we do not report its results. 

\paragraph{Comparison of our bounds with and without different initialization and regularization:}
We present an ablation study in Table~\ref{tab:init_prior_app}.
To see if the added flexibility of our setting is
indeed responsible for the improved results rather
than the different objective compared to prior
work, we compare our results with the case 
that the two distributions are equal
when we use the distribution learned for the regularization for the initialization as well.
As one can see, using different distributions for 
initialization and regularization reduce the error 
in the \emph{permuted labels} task, but for \emph{shuffled pixels} they stay the same.

\begin{table*}[h!]
\centering
\begin{tabular}{ |c||c|c|   }
 \hline
 Bound & Shuffled Pixel & Permuted Label\\
 \hline 
 initialization identical to prior ($\rho_0=\rho_1$)   & \np{10.1} $\pm$ \np{1.25} & \np{15.5} $\pm$ \np{4.27} \\
 initialization can differ from prior   & \np{9.9} $\pm$ \np{1.14} & \np{7.91} $\pm$ \np{1.7} \\
 \hline
\end{tabular}
\caption{The ablation study to confirm that the gained performance is due to separate initialization and regularization. We compare our method with the case if we use the distributions in the end of our training procedure both as the initialization and regularization. As one can see, having different distributions leads to better performance in one of the tasks.}
\label{tab:init_prior_app}
\end{table*}

\end{document}